\documentclass[a4paper, twocolumn]{article}
\usepackage[T1]{fontenc}
\usepackage[a4paper, total={170mm,257mm}]{geometry}
\usepackage{etoolbox}
\makeatletter
\providecommand{\institute}[1]{% add institute to \maketitle
  \apptocmd{\@author}{\end{tabular}
    \par
    \begin{tabular}[t]{c}
    #1}{}{}
}
\makeatother
\usepackage{cite}
\usepackage{amsmath,amssymb,amsfonts}
\usepackage{algorithmic}
\usepackage{graphicx}
\usepackage{textcomp}
\usepackage{xcolor}

\usepackage{subcaption}
\usepackage{apxproof}
\usepackage{booktabs}
\usepackage{apxproof}
\usepackage{tikz}
\usetikzlibrary{positioning,shadows,shapes.symbols,arrows.meta}

\usepackage{todonotes}
\usepackage{url}

\renewcommand{\P}{\mathbb{P}}
\newcommand{\D}{\mathcal{D}}
\newcommand{\X}{\mathcal{X}}
\newcommand{\T}{\mathcal{T}}
\newcommand{\R}{\mathbb{R}}
\newcommand{\Z}{\mathbb{Z}}
\newcommand{\E}{\mathbb{E}}
\newcommand{\1}{\mathbf{1}}
\newcommand{\pa}{\textnormal{Pa}}

\newcommand{\indep}{{\perp\!\!\!\perp}}
\newcommand{\argmax}{\text{arg}\max}

\newcommand{\hiddencause}{H}
\newcommand{\demand}{D}
\newcommand{\observable}{O}
\newcommand{\anomalyI}{A}
\newcommand{\anomalyII}{S}
\newcommand{\measurement}{M}

\newtheorem{theorem}{Theorem}
\newtheorem{lemma}{Lemma}

\theoremstyle{definition}
\newtheorem{definition}{Definition}

\theoremstyle{remark}

\usepackage{pgfplots}
\usepackage{pgfplotstable}

\makeatletter
\pgfplotsset{
    /pgfplots/table/omit header/.style={%
        /pgfplots/table/typeset cell/.append code={%
            \ifnum\c@pgfplotstable@rowindex=-1
                \pgfkeyslet{/pgfplots/table/@cell content}\pgfutil@empty%
            \fi
        }
    }
}
\makeatother

\usepackage[binary-units,group-separator={,}]{siunitx}

\newcolumntype{N}{S[table-format=2.2,round-mode=places,round-precision=2]}%,round-mode=off
\newcolumntype{M}{S[table-format=2.0,round-mode=places,round-precision=0]}%,round-mode=off
\newcolumntype{O}{@{\,}c@{\,}S[table-format=2.2,round-mode=places,round-precision=2]@{}c@{\ }c}
\newcolumntype{P}{S[table-format=1.4,round-mode=places,round-precision=4]}

\newcommand{\results}[2][]{
\setlength{\tabcolsep}{0.5em}
    \pgfplotstabletypeset[
    col sep=comma,
    string type,
    columns/method/.style={column type={l|},column name=\multicolumn{1}{l|}{method}},
    columns/l_dist/.style={column type=O,column name=},
    columns/l_dist_med/.style={column type=M,column name=\multicolumn{4}{c}{distance ($-$)}},
    columns/l_dist_mean/.style={column type=N,column name={}},
    columns/l_closer/.style={column type=O,column name={}},
    columns/l_closer_med/.style={column type=M,column name=\multicolumn{4}{c}{$\#$closer ($-$)}},
    columns/l_closer_mean/.style={column type=N,column name={}},
    columns/l_reldist/.style={column type=O,column name={}},
    columns/l_reldist_med/.style={column type=M,column name=\multicolumn{4}{c}{rel. dist ($-$)}},
    columns/l_reldist_mean/.style={column type=N,column name={}},
    columns/l_best3/.style={column type=O|,column name={}},
    columns/l_best3_med/.style={column type=M,column name=\multicolumn{4}{c}{best 3 ($+$)}},
    columns/l_best3_mean/.style={column type=N,column name={}},
    columns/s_dist/.style={column type=O,column name={}},
    columns/s_dist_med/.style={column type=M,column name=\multicolumn{4}{c}{distance ($-$)}},
    columns/s_dist_mean/.style={column type=N,column name={}},
    columns/s_closer/.style={column type=O,column name={}},
    columns/s_closer_med/.style={column type=M,column name=\multicolumn{4}{c}{closer ($-$)}},
    columns/s_closer_mean/.style={column type=N,column name={}},
    columns/recall/.style={column type=P,column name={\multicolumn{1}{c}{recall}}},
    columns/precision/.style={column type=P,column name=\multicolumn{1}{c}{precision}},
    columns/f1/.style={column type=P,column name={\multicolumn{1}{c}{F1}}},
    columns={method,l_dist_med,l_dist_mean,l_dist,l_closer_med,l_closer_mean,l_closer,l_reldist_med,l_reldist_mean,l_reldist,l_best3_med,l_best3_mean,l_best3,recall,precision,f1},%,s_reldist_med,s_reldist_mean,s_reldist
    every head row/.style={before row={\multicolumn{1}{c|}{}&\multicolumn{24}{c|}{type I: leakages}&\multicolumn{3}{c}{type II: sensor faults}\\},after row={\multicolumn{1}{c|}{}&m&$\mu$&$\pm$&$\sigma$&&&m&$\mu$&$\pm$&$\sigma$&&&m&$\mu$&$\pm$&$\sigma$&&&m&$\mu$&$\pm$&$\sigma$&&&\\\midrule}},
    every row no 2/.append style={after row=\midrule},
    #1,
    create on use/l_dist/.style={
      create col/assign/.code={%
        \ifnum\pdfstrcmp{\thisrow{l_dist_mean}}{}=0
        \edef\entry{--&&}
        \else
        \edef\entry{$\pm$&\thisrow{l_dist_std}&&}
        \fi
        \pgfkeyslet{/pgfplots/table/create col/next content}\entry
      }
    },
    create on use/l_closer/.style={
      create col/assign/.code={%
        \ifnum\pdfstrcmp{\thisrow{l_closer_mean}}{}=0
        \edef\entry{--&&}
        \else
        \edef\entry{$\pm$&\thisrow{l_closer_std}&&}
        \fi
        \pgfkeyslet{/pgfplots/table/create col/next content}\entry
      }
    },
    create on use/l_reldist/.style={
      create col/assign/.code={%
        \ifnum\pdfstrcmp{\thisrow{l_reldist_mean}}{}=0
        \edef\entry{--&&}
        \else
        \edef\entry{$\pm$&\thisrow{l_reldist_std}&&}
        \fi
        \pgfkeyslet{/pgfplots/table/create col/next content}\entry
      }
    },
    create on use/l_best3/.style={
      create col/assign/.code={%
        \ifnum\pdfstrcmp{\thisrow{l_best3_mean}}{}=0
        \edef\entry{--&&}
        \else
        \edef\entry{$\pm$&\thisrow{l_best3_std}&&}
        \fi
        \pgfkeyslet{/pgfplots/table/create col/next content}\entry
      }
    },
    create on use/s_dist/.style={
      create col/assign/.code={%
        \ifnum\pdfstrcmp{\thisrow{s_dist_mean}}{}=0
        \edef\entry{--&&}
        \else
        \edef\entry{$\pm$&\thisrow{s_dist_std}&&}
        \fi
        \pgfkeyslet{/pgfplots/table/create col/next content}\entry
      }
    },
    create on use/s_closer/.style={
      create col/assign/.code={%
        \ifnum\pdfstrcmp{\thisrow{s_closer_mean}}{}=0
        \edef\entry{--&&}
        \else
        \edef\entry{$\pm$&\thisrow{s_closer_std}&&}
        \fi
        \pgfkeyslet{/pgfplots/table/create col/next content}\entry
      }
    },
    create on use/s_reldist/.style={
      create col/assign/.code={\edef\entry{$\pm$&\thisrow{s_reldist_std}&&}%
        \pgfkeyslet{/pgfplots/table/create col/next content}\entry
      }
    }
    ]{#2}
}

\newcommand{\resultsElec}[2][]{
\setlength{\tabcolsep}{0.5em}
    \pgfplotstabletypeset[
    col sep=comma,
    string type,
    columns/method/.style={column type={l|},column name=\multicolumn{1}{l|}{method}},
    columns/s_dist/.style={column type=O,column name={}},
    columns/s_dist_med/.style={column type=M,column name=\multicolumn{4}{c}{distance ($-$)}},
    columns/s_dist_mean/.style={column type=N,column name={}},
    columns/s_closer/.style={column type=O,column name={}},
    columns/s_closer_med/.style={column type=M,column name=\multicolumn{4}{c}{closer ($-$)}},
    columns/s_closer_mean/.style={column type=N,column name={}},
    columns={method,s_dist_med,s_dist_mean,s_dist,s_closer_med,s_closer_mean,s_closer},
    every head row/.style={before row={\multicolumn{1}{c|}{}&\multicolumn{12}{c}{type II: sensor faults}\\},after row={\multicolumn{1}{c|}{}&m&$\mu$&$\pm$&$\sigma$&&&m&$\mu$&$\pm$&$\sigma$\\\midrule}},
    every row no 2/.append style={after row=\midrule},
    #1,
    create on use/s_dist/.style={
      create col/assign/.code={%
        \ifnum\pdfstrcmp{\thisrow{s_dist_mean}}{}=0
        \edef\entry{--&&}
        \else
        \edef\entry{$\pm$&\thisrow{s_dist_std}&&}
        \fi
        \pgfkeyslet{/pgfplots/table/create col/next content}\entry
      }
    },
    create on use/s_closer/.style={
      create col/assign/.code={%
        \ifnum\pdfstrcmp{\thisrow{s_closer_mean}}{}=0
        \edef\entry{--&&}
        \else
        \edef\entry{$\pm$&\thisrow{s_closer_std}&&}
        \fi
        \pgfkeyslet{/pgfplots/table/create col/next content}\entry
      }
    },
    create on use/s_reldist/.style={
      create col/assign/.code={\edef\entry{$\pm$&\thisrow{s_reldist_std}&&}%
        \pgfkeyslet{/pgfplots/table/create col/next content}\entry
      }
    }
    ]{#2}
}

\newcommand{\resultsElecAcc}[2][]{
\setlength{\tabcolsep}{0.5em}
    \pgfplotstabletypeset[
    col sep=comma,
    string type,
    columns/method/.style={column type={l|},column name=\multicolumn{1}{l|}{method}},
    columns/recall/.style={column type=P,column name={\multicolumn{1}{c}{recall}}},
    columns/precision/.style={column type=P,column name=\multicolumn{1}{c}{precision}},
    columns/f1/.style={column type=P,column name={\multicolumn{1}{c}{f1}}},
    columns={method,recall,precision,f1},
    every head row/.style={after row={\midrule}},
    every row no 2/.append style={after row=\midrule},
    #1,
    create on use/s_dist/.style={
      create col/assign/.code={%
        \ifnum\pdfstrcmp{\thisrow{s_dist_mean}}{}=0
        \edef\entry{--&&}
        \else
        \edef\entry{$\pm$&\thisrow{s_dist_std}&&}
        \fi
        \pgfkeyslet{/pgfplots/table/create col/next content}\entry
      }
    },
    create on use/s_closer/.style={
      create col/assign/.code={%
        \ifnum\pdfstrcmp{\thisrow{s_closer_mean}}{}=0
        \edef\entry{--&&}
        \else
        \edef\entry{$\pm$&\thisrow{s_closer_std}&&}
        \fi
        \pgfkeyslet{/pgfplots/table/create col/next content}\entry
      }
    },
    create on use/s_reldist/.style={
      create col/assign/.code={\edef\entry{$\pm$&\thisrow{s_reldist_std}&&}%
        \pgfkeyslet{/pgfplots/table/create col/next content}\entry
      }
    }
    ]{#2}
}
\usepackage{tikz}
\usepackage{pgfplots}
\pgfplotsset{compat=1.8}
\usetikzlibrary{arrows.meta,patterns,decorations.text,fit,calc}
\usepgfplotslibrary{statistics}

\definecolor{tangored}{HTML}{CC0000}
\definecolor{tangoplum}{HTML}{75507B}
\definecolor{tangoblue}{HTML}{3465A4}
\definecolor{tangoorange}{HTML}{F57900}
\definecolor{tangogreen}{HTML}{73D216}
\definecolor{tangobutter}{HTML}{EDD400}

\definecolor{aluminium1}{HTML}{EEEEEC}
\definecolor{aluminium2}{HTML}{D3D7CF}
\definecolor{aluminium3}{HTML}{BABDB6}
\definecolor{aluminium4}{HTML}{888A85}
\definecolor{aluminium5}{HTML}{555753}
\definecolor{aluminium6}{HTML}{2E3436}

\colorlet{class0c}{aluminium4}
\colorlet{class1c}{tangoplum}
\colorlet{class2c}{tangogreen}
\colorlet{class3c}{tangoorange}
\colorlet{class4c}{tangoblue}

\tikzstyle{classinv}=[draw=none, mark size=0pt]
\tikzstyle{class0}=[fill=class0c, draw=aluminium6,mark size=1pt]
\tikzstyle{class1}=[fill=class1c, draw=class1c!80!black,mark size=1pt]
\tikzstyle{class2}=[fill=class2c, draw=class2c!80!black,mark size=1pt]
\def\xshift{3.5}
\def\yshift{8}

\pgfmathdeclarefunction{invgauss}{2}{%
	\pgfmathparse{sqrt(-2*ln(#1))*cos(deg(2*pi*#2))}%
}

\pgfplotsset{ticks=none,
	yticklabels=\empty,
	xticklabels=\empty,
	width=0.15\textwidth,
	height=0.15\textwidth,
	label style={font=\labelsize},
}

\newcommand{\plotdataset}[4][1]{
	\pgfmathsetseed{42}
	\addplot [#2,only marks, samples=10] ({invgauss(rnd,rnd)},{invgauss(rnd,rnd)});
	\addplot [#2,only marks, samples=5] ({1.5*invgauss(rnd,rnd)+#4},{invgauss(rnd,rnd)+#1*\yshift});
	\addplot [#3, only marks, samples=10] ({invgauss(rnd,rnd)+2*#4},{invgauss(rnd,rnd)});
	\addplot [#3, only marks, samples=5] ({1.5*invgauss(rnd,rnd)+#4},{invgauss(rnd,rnd)+#1*\yshift});
}

\newcommand{\plotSegment}[1][]{
	\pattern[pattern=north east lines, pattern color=class1c] (axis cs:-5,-5)--(axis cs:-5,0.5*\yshift)--(axis cs:\xshift,0.5*\yshift)--(axis cs:\xshift,-5)--cycle;
	\pattern[pattern=north west lines, pattern color=class2c] (axis cs:\xshift,-5)--(axis cs:\xshift,0.5*\yshift)--(axis cs:4*\xshift,0.5*\yshift)--(axis cs:4*\xshift,-5)--cycle;
	\ifthenelse{\equal{#1}{}}{}{
		\pattern[pattern=vertical lines, pattern color=class0c] (axis cs:-5,0.5*\yshift)--(axis cs:-5,2*\yshift)--(axis cs:4*\xshift,2*\yshift)--(axis cs:4*\xshift,0.5*\yshift)--cycle;
	}
}

\newcommand{\labelsize}{\scriptsize}

\newcommand{\tikzLocalization}{
\begin{tikzpicture}[
	node distance=5em,
	from/.style={{Stealth[length=0.8em, round]}-,double,shorten >=-0.2em,shorten <=-0.2em},
	towards/.style={-{Stealth[length=0.8em, round]},double,shorten >=-0.2em,shorten <=-0.2em},
	]
		\node(beforedrift) {%
			\begin{tikzpicture}
			\begin{axis}[xlabel={before drift}]
				 \plotdataset{class1}{classinv}{\xshift}
			\end{axis}
			\end{tikzpicture}
		};
		\node(afterdrift) [node distance=0em, below=of beforedrift] {%	
			\begin{tikzpicture}
			\begin{axis}[xlabel={after drift}]
				 \plotdataset{classinv}{class2}{\xshift}
			\end{axis}
			\end{tikzpicture}
		};
	\draw[decorate,decoration={brace,amplitude=1.5em,mirror,raise=-0.3em}] (beforedrift.north west) -- node[xshift=-1.2em](leftdrift){} (afterdrift.south west);
	\draw[decorate,decoration={brace,amplitude=1.5em,raise=-0.3em}] (beforedrift.north east) -- node[xshift=1.2em](rightdrift){} (afterdrift.south east);

	\node(stream) [left=of leftdrift] {%
		\begin{tikzpicture}
			\begin{axis}[xlabel={}]
				 \plotdataset{class0}{class0}{\xshift}
			\end{axis}
		\end{tikzpicture}
	};% edge[towards]
   \draw (stream.east)edge[-{Latex[length=3mm]}]
		[postaction={decoration={text along path, text align=center,text={|\labelsize |drift ||},
			raise=0.2em,},decorate}]
		[postaction={decoration={text along path, text align=center,text={|\labelsize |detection ||},
		raise=-0.7em,}, decorate}]
 (leftdrift.west);
 	\node[node distance=0em,below=of stream] {\labelsize stream};
	\node(driftingarea) [right=of rightdrift] {
		\begin{tikzpicture}
		\begin{axis}
			\plotdataset{class1}{class2}{\xshift}
			\plotSegment
		\end{axis}
	\end{tikzpicture}
	} ;%edge [from] node[above]{\labelsize train} 
 \draw (rightdrift.east)edge[-{Latex[length=3mm]}]
		[postaction={decoration={text along path, text align=center,text={|\labelsize |train ||},
			raise=0.2em,},decorate}]
		[postaction={decoration={text along path, text align=center,text={|\labelsize |model ||},
		raise=-0.7em,}, decorate}]
 (driftingarea.west);

 \node (pfi) [right=of driftingarea] {
							\begin{tikzpicture}
								\begin{axis}[
									boxplot/draw direction=x,
									ymin=0,ymax=5,%ymajorticks=true,
									cycle list={{class1c},{class2c}},
									tick align=inside,
									ytick={1,2,3,4},
									yticklabels={$z_2$,$z_1$,$y$,$x$},
									]
									\addplot+[fill,fill opacity=0.2,
									boxplot prepared={
										median=2,
										upper quartile=2.35,
										lower quartile=1.5,
										upper whisker=3.4,
										lower whisker=1.1
									},
									] coordinates {};
									\addplot+[fill,fill opacity=0.2,
									boxplot prepared={
										median=2.59,
										upper quartile=3.35,
										lower quartile=2,
										upper whisker=4.4,
										lower whisker=1.1
									},
									] coordinates {};
									\addplot+[fill,fill opacity=0.2,
									boxplot prepared={
										median=5.57,
										upper quartile=6.93,
										lower quartile=4.93,
										upper whisker=8.68,
										lower whisker=3.03
									},
									] coordinates {};
									\addplot+[fill,fill opacity=0.2,
									boxplot prepared={
										median=6.57,
										upper quartile=7.93,
										lower quartile=5.93,
										upper whisker=10.68,
										lower whisker=4.03
									},
									] coordinates {};
								\end{axis}
							\end{tikzpicture}
							
					};
        \node[node distance=0em,below=of pfi] {\labelsize feature importances};
        \node[node distance=0em,below=of driftingarea] {\labelsize drift localization};
     \draw (driftingarea.east)edge[-{Latex[length=3mm]}]
		[postaction={decoration={text along path, text align=center,text={|\labelsize |model ||},
			raise=0.2em,},decorate}]
		[postaction={decoration={text along path, text align=center,text={|\labelsize |explanation ||},
		raise=-0.7em,}, decorate}]
 (pfi.west);
	\end{tikzpicture}
 }

\def\BibTeX{{\rm B\kern-.05em{\sc i\kern-.025em b}\kern-.08em
    T\kern-.1667em\lower.7ex\hbox{E}\kern-.125emX}}

\begin{document}

\title{Localizing Anomalies in Critical Infrastructure using Model-Based Drift Explanations\thanks{We gratefully acknowledge funding from the European Research Council (ERC) under the ERC Synergy Grant Water-Futures (Grant agreement No. 951424) and funding in the frame of SAIL which is funded by the Ministry of Culture and Science of the State of North Rhine-Westphalia under the grant no NW21-059B.\\\textsuperscript{**}Authors contributed equally.}}

\author{Valerie Vaquet$^{1**}$, Fabian Hinder$^{1**}$, Jonas Vaquet$^1$, Kathrin Lammers$^1$,\\Lars Quakernack$^2$, and Barbara Hammer$^1$}
\institute{$^1$Machine Learning Group, Bielefeld University, Bielefeld, Germany\\
$^2$Institute for Technical Energy Systems, Hochschule Bielefeld –- \\University of Applied
Sciences and Arts
Bielefeld, Germany}
\date{February 2024}

\maketitle
% Each paper is limited to 8 pages, including figures, tables, and references. A maximum of two extra pages per paper is allowed (i.e, up to 10 pages), at an additional charge of $100 per extra page.
\begin{abstract}
% Facing climate change the already limited availability of drinking water will decrease in the future rendering drinking water an increasingly scarce resource. Considerable amounts of it are lost through leakages in water transportation and distribution networks. Leakage detection and localization are challenging problems due to the complex interactions and changing demands in water distribution networks. 
% % remove next sentence if too much
% Especially small leakages are hard to pinpoint yet their localization is vital to avoid water loss over long periods. While there exist different approaches to solving the tasks of leakage detection and localization, they rely on various information about the system, e.g. real-time demand measurements and the precise network topology, which is an unrealistic assumption in many real-world scenarios. In contrast, this work attempts leakage localization using pressure measurements only. For this purpose, first, leakages in the water distribution network are modeled employing Bayesian networks, and the system dynamics are analyzed. We then show how the problem is connected to and can be considered through the lens of concept drift. In particular, we argue that model-based explanations of concept drift are a promising tool for localizing leakages given limited information about the network. The methodology is experimentally evaluated using realistic benchmark scenarios.
Facing climate change, the already limited availability of drinking water will decrease in the future rendering drinking water an increasingly scarce resource. Considerable amounts of it are lost through leakages in water transportation and distribution networks. Thus, anomaly detection and localization, in particular for leakages, are crucial but challenging tasks due to the complex interactions and changing demands in water distribution networks. 
In this work, we analyze the effects of anomalies on the dynamics of critical infrastructure systems by modeling the networks employing Bayesian networks. We then discuss how the problem is connected to and can be considered through the lens of concept drift. In particular, we argue that model-based explanations of concept drift are a promising tool for localizing anomalies given limited information about the network. The methodology is experimentally evaluated using realistic benchmark scenarios. To showcase that our methodology applies to critical infrastructure more generally, in addition to considering leakages and sensor faults in water systems, we showcase the suitability of the derived technique to localize sensor faults in power systems.
\end{abstract}

% \begin{IEEEkeywords}
\paragraph{Keywords}
Water Distribution Networks, Leakage Localization, Anomaly Localization, Concept Drift, Explainable AI, Model-Based Drift Explanations.
% \end{IEEEkeywords}

\section{Introduction}\
Clean and safe drinking water is a scarce resource in many areas. Almost 80\% of the world's population is classified as having high levels of threat in water security~\cite{vorosmarty2010global}. This will aggravate in the future as due to climate change the already limited water resources will become more restricted~\cite{rodell2018emerging}.
Currently, across Europe, considerable amounts of drinking water are lost due to leakages in the system\footnote{\url{https://www.eureau.org/resources/publications/1460-eureau-data-report-2017-1/file}}. The issue has been identified by the European Union, which put the topic on the political agenda recently\footnote{\url{https://eur-lex.europa.eu/legal-content/EN/TXT/HTML/?uri=CELEX:32020L2184&from=ES#d1e40-1-1}}.

To ensure a reliable drinking water supply, there is a need for reliable, safe, and efficient water distribution networks (WDNs). In addition to avoiding water losses, a crucial requirement is to ensure the quality of the drinking water, e.g. to avoid the spreading and growth of bacteria and other contaminants. Leakages pose a major risk to water quality as they make it possible for unwanted substances to enter the water system~\cite{eliades_contamination_2023}. Thus, monitoring the system for leakages is an efficient tool to avoid both water loss and contamination~\cite{eliades_fault_2010}. Besides, monitoring for other anomalies like sensor faults is required to avoid errors in further control and supervision downstream tasks~\cite{faults-water}.

Due to complex network dynamics and changing demand patterns detecting and localizing anomalies are challenging tasks~\cite{vrachimis_battle_2022}. This is aggravated by the fact, that the available data is very limited~\cite{eggimann_potential_2017}. Usually, the precise network topology remains unknown or the documentation contains errors. As smart water meter technologies are not widely distributed there is no real-time demand information~\cite{smartmeters}. This leaves a set of scarce pressure and possibly flow measurements. When considering historical recordings the presence and exact timings of possible leakages are frequently unknown as especially smaller leakages might not have been detected. 

%evtl streichen
% Detecting and localizing small leakages are particularly hard tasks as the changes induced by this anomaly easily get lost in the complex dynamics of the system. However, detecting small leakages is of particular relevance, as water losses and potential contamination over extended time periods are expected~\cite{lambert_accounting_1994}. In contrast, larger leakages are usually detected more easily by monitoring tools or other means. For example, they are exposed to water accumulating in streets and are usually repaired within a short time window resulting in less water loss in total.

Many of the state-of-the-art schemes for leakage detection and localization are based on building a hydraulic model replicating a specific system. This requires a lot of problem-specific information like the precise network topology and real-time demands~\cite{vrachimis_battle_2022}. Since frequently this information is not available, this strongly limits the applicability and generalization of current solutions. Besides, whenever the real system changes the need to adapt the simulation arises.
Recently, considerable research on applying machine learning techniques to leakage detection has been conducted \cite{hu_survey_2018,wu_review_2017}. However, machine learning approaches tackling the task of leakage localization are still limited and frequently rely on network topologies and historic leakage-free data.

\begin{figure*}[t]
\centering
\begin{minipage}{0.45\textwidth}
    \centering
    \includegraphics[width=0.8\textwidth]{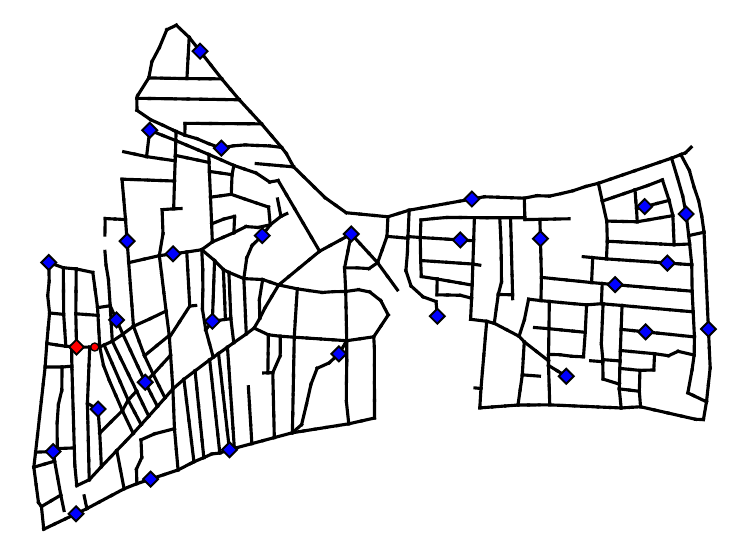}
    \subcaption{L-Town topology. Diamonds mark the position of pressure sensors, red dot marks the leakage position.\label{fig:ltwon:map}}
\end{minipage}%\hfill
\begin{minipage}{0.45\textwidth}
    \centering
    \includegraphics[width=0.8\textwidth]{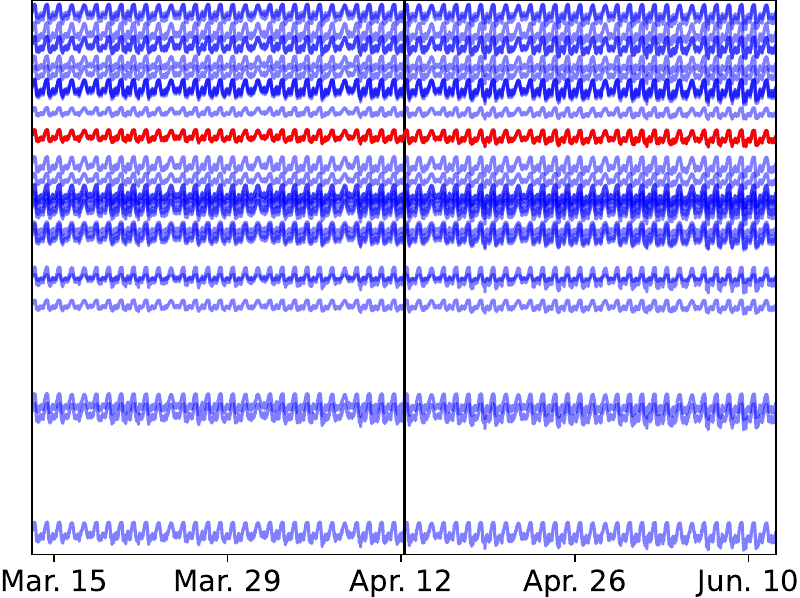}
    \subcaption{Pressures at sensors in Figure~\ref{fig:ltwon:map} (colors match), black line marks occurrence time of leakage.\label{fig:ltwon:pressure}}
\end{minipage}
\caption{L-Town~\cite{vrachimis_battle_2022} topology and pressure value example. Small leakages cannot be found by the human eye even if highlighted. }
\label{fig:ltwon}
\end{figure*}

In this work, we analyze anomalies by describing the critical infrastructure and two conceptionally different types of anomalies utilizing dynamic Bayesian Networks. Based on this modeling, we derive a definition for anomaly detection and localization. Our analysis justifies the choice of many methodologies used for leakage detection. Concerning the task of anomaly localization, we propose to apply a model-based drift explanation~\cite{neucomp}. This methodology has the advantage that, in contrast to most state-of-the-art solutions, it does not require precise topological information, real-time demands, or historical leakage-free data, and it can be directly applied to a new network without requiring retraining or modeling the new network. While our work mainly focuses on WDNs, it applies to different critical infrastructure systems which can be modeled as graphs. We exemplary showcase this by considering an electrical grid (EG) in our experiments.

This paper is structured as follows. First, a brief introduction to critical infrastructure is provided (Section~\ref{sec:wdns}). Afterward, in Section~\ref{sec:causal-model} we model anomalies in critical infrastructure utilizing dynamic Bayesian networks and propose a formal definition of anomaly detection and localization based on our modeling. Based on this we derive the proposed methodology (Section~\ref{sec:methodology}). Finally, we conclude our paper by presenting an experimental evaluation of the proposed model-based drift explanation methodology (Section~\ref{sec:exp}) and discussing our findings (Section~\ref{sec:discussion}).

% In order to ensure the water supply of the people facing these issues efficient and reliable water distribution systems are required. While efficiently transporting and distributing water, the objective is to avoid water losses and limit the risk of contamination. One particular relevant task to meet these goals is to implement a monitoring system that is reliably detecting and localizing leakages under changing dynamics in the network. By quickly accommodating for leakages water loss can be limited and the risk of contamination due to undesirable substances entering the system through leakages can be reduced. Usually detecting and localizing small leakages is a much harder yet vital task. Over time the amount of lost water accumulates and usually, the water lost through smaller leaks seeps away instead of accumulating on the surface as in case of larger leakages or pipe breaks.

% In this work, we consider leakages in the water distribution system as concept drift. We explore the potential of model-based drift explanations \todo{cite neucom} to localize small leakages. By modeling the system as a Bayesian network we motivate our approach  theoretically. \todo{discuss what we exactly state as our contribution/RQ}

% \section{Leakage Detection in Water Distribution Networks\label{sec:wdns}}
\section{Modeling Critical Infrastructure\label{sec:wdns}}

% This section briefly discusses how WDNs are usually modeled, and which network benchmarks are available. Besides, it introduces the standard monitoring setup and summarizes the body of related work on leakage detection and localization.

% \subsection{Modeling Critical Infrastructure}
As discussed before, a substantial part of the critical infrastructure focuses on supplying citizens and industries with the goods necessary for daily life and economies to work efficiently -- for example, WDNs, EGs, and gas networks cover the basic needs of the citizens. Besides, transportation and telecommunication networks are crucial to ensure economic wealth. All these instances share one key property: As exemplary visualized in Fig. \ref{fig:ltwon:map} for a WDN, they can be modeled as graphs consisting of edges representing pipes, cables, streets, etc., and nodes representing their junctions or additional components like pumps, transformers, etc. Next to describing the overall topology, the properties of the components can be described by adding edge and node features. Depending on the type of infrastructure, the exact topology and the edge and node features might be unknown. For example, considering WDNs the exact diameters of certain pipes are frequently unknown and elevation levels within the network are rarely measured due to the associated cost~\cite{eggimann_potential_2017}. Besides, usually, networks in critical infrastructure evolve due to extensions and naturally occurring changes, e.g. the roughness of pipes might change due to reactions with substances in the water.

Facilitated by technological advancement and decreasing costs, critical infrastructure is increasingly equipped with sensor technologies~\cite{eggimann_potential_2017}. Different types of sensors can be installed at both nodes and edges throughout the system. However, as their installation and maintenance are costly, frequently only sparse real-time sensor measurements are available to monitor and control the networks. As mentioned before, in many systems, some kind of product is supplied to the end-user. Thus, the demands are heavily impacting the dynamics in the network as can be seen in the water pressure measurements in Fig. \ref{fig:ltwon:pressure}. While it is technologically possible to measure the real-time demand of the end-users by so-called smart meters, usually this is not done exhaustively due to the associated cost and privacy concerns \cite{smartmeters}. % Im elektrischen Netz ist der Smart Meter roll out in manchen Ländern schon sehr weit. Auch In Deutschland ist/wird der Smart Meter Einbau verpflichtent.

% \subsection{Anomalies in Critical Infrastructure}
% \todo[inline]{look and cite resources }%https://home-affairs.ec.europa.eu/pages/page/critical-infrastructure_en
Since damage, destruction, and disruption of critical infrastructure pose major risks to the safety and well-being of citizens, monitoring for anomalous behavior is mandatory. Anomalies might be caused by a deterioration of materials, natural disasters, or different types of malicious behavior, e.g. terrorism or criminal activity. Monitoring efforts can be divided into multiple steps: First, anomalous behavior needs to be detected. Afterward, the event has to be analyzed further to decide on an appropriate reaction. This analysis might contain identifying the type of anomaly, its magnitude, and its location~\cite{reppa_sensor_2016}.  

The focus of this work is anomaly localization. However, our analysis of the role of anomalies in critical infrastructure, which we will present in the next section, justifies the choice of many methods for anomaly detection.

\section{A Causal Model for Modeling Anomalies in Critical Infrastructure\label{sec:causal-model}}

% \begin{figure}[t]
%     \centering
%     \begin{minipage}{0.45\textwidth}
%     \includegraphics[width=0.95\textwidth]{src/img/dynBayes}%
%     \subcaption{Complete dynamic Bayesian network. Stacked circles indicate temporal self-dependency (compare Figure~\ref{fig:bayes-net:presures}).}
%     \label{fig:bayes-net:complete}
%     \end{minipage}
%     \begin{minipage}{0.45\textwidth}
%     \includegraphics[width=0.75\textwidth]{src/img/Pipe}%
%     \subcaption{Simple pipe-network example.}
%     \label{fig:bayes-net:pipes}
%     \;\\
%     \includegraphics[width=0.95\textwidth]{src/img/dynBayesPipe}% 
%     \subcaption{Dynamic Bayesian sub-network for presure dynamics of the pipe network presented in Figure~\ref{fig:bayes-net:pipes}.}
%     \label{fig:bayes-net:presures}
%     \end{minipage}
%     \caption{Visualization of the modeling by a dynamic Bayesian network.}
%     \label{fig:bayes-net}
% \end{figure}

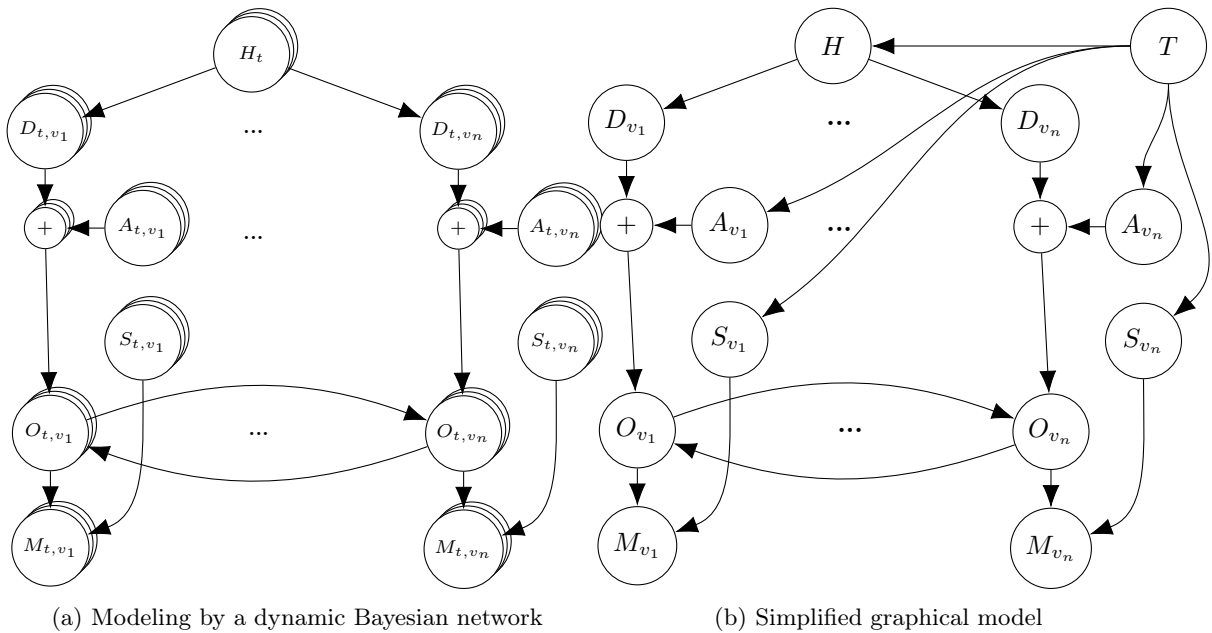
\begin{figure*}[t]
\centering
\begin{subfigure}{0.45\textwidth}
    
    \centering
    \scriptsize
    \pgfdeclarelayer{bg}    % declare background layer
    \pgfsetlayers{bg,main}  % set the order of the layers (main is the standard layer)
    \begin{tikzpicture}[minimum height=1cm, every shadow/.style={fill=white,shadow scale=1, shadow xshift=.5ex, shadow yshift=-.5ex, every shadow}, every node/.style={fill=white, draw,double copy shadow}, node distance=0.5cm]%auto, thick, node distance=1cm, 
      
        \node at (0,0) [circle] (ht){$H_t$};
        \node[below left = 0.3cm and 2cm of ht]  [circle, draw] (d1){$D_{t,v_1}$};
        \node[below right =0.3cm and 2cm of ht]  [circle, draw] (dn){$D_{t,v_n}$};

        \node[below = of d1]  [circle, draw, minimum height=0.5cm] (sum){$+$};
        \node[below = of dn]  [circle, draw, minimum height=0.5cm] (sumn){$+$};

        \node[right = of sum]  [circle, draw] (l1){$\anomalyI_{t,v_1}$};
        \node[right = of sumn]  [circle, draw] (ln){$\anomalyI_{t,v_n}$};

        \node[below = of l1]  [circle, draw] (s1){$\anomalyII_{t,v_1}$};
        \node[below = of ln]  [circle, draw] (sn){$\anomalyII_{t,v_n}$};

        \node[below left = 2cm and 0.5cm of l1]  [circle, draw] (p1){$\observable_{t,v_1}$};
        \node[below left = 2cm and 0.5cm of ln]  [circle, draw] (pn){$\observable_{t,v_n}$};

        \node[below  = of p1]  [circle, draw] (m1){$\measurement_{t,v_1}$};
        \node[below  =  of pn]  [circle, draw] (mn){$\measurement_{t,v_n}$};
      
        \begin{pgfonlayer}{bg} 
        \node[right = 2cm of d1, draw=white] (dots) {\textbf{...}};      
        \node[below = 0.4cm of dots, draw=white] (dots2) {\textbf{...}};      
        \node[right = 2cm of p1, draw=white] (dots3) {\textbf{...}};
        \end{pgfonlayer}
      
        \draw[-{Latex[length=3mm]}] (ht) -- (d1);
        \draw[-{Latex[length=3mm]}] (ht) -- (dn);
            
      \draw[-{Latex[length=3mm]}] (d1) -- (sum);
      \draw[-{Latex[length=3mm]}] (dn) -- (sumn);
      \draw[-{Latex[length=3mm]}] (l1) -- (sum);
      \draw[-{Latex[length=3mm]}] (ln) -- (sumn);
      \draw[-{Latex[length=3mm]}] (sum) -- (p1);
      \draw[-{Latex[length=3mm]}] (sumn) -- (pn);
      
      \draw[-{Latex[length=3mm]}] (s1) edge [out=270,in=20]  (m1);
      \draw[-{Latex[length=3mm]}] (sn) edge [out=270,in=20] (mn);
      \draw[-{Latex[length=3mm]}] (p1) -- (m1);
      \draw[-{Latex[length=3mm]}] (pn) -- (mn);

      \draw[-{Latex[length=3mm]}] (p1) edge [out=20,in=160]  (pn);
      \draw[-{Latex[length=3mm]}] (pn) edge [out=200,in=340] (p1);

    \end{tikzpicture}
    \caption{Modeling by a dynamic Bayesian network}
    \label{fig:bayes-net}
\end{subfigure}%
\begin{subfigure}{0.45\textwidth}
    \centering
    \pgfdeclarelayer{bg}    % declare background layer
\pgfsetlayers{bg,main}  % set the order of the layers (main is the standard layer)
    \begin{tikzpicture}[minimum height=1cm, every node/.style={fill=white, draw}, node distance=0.5cm]%auto, thick, node distance=1cm, 
      
        \node at (0,0) [circle] (ht){$H$};
        \node[right =3.4cm of ht]  [circle, draw] (t){$T$};
        \node[below left = 0.3cm and 2cm of ht]  [circle, draw] (d1){$D_{v_1}$};
        \node[below right =0.3cm and 2cm of ht]  [circle, draw] (dn){$D_{v_n}$};

        \node[below = of d1]  [circle, draw, minimum height=0.5cm] (sum){$+$};
        \node[below = of dn]  [circle, draw, minimum height=0.5cm] (sumn){$+$};

        \node[right = of sum]  [circle, draw] (l1){$\anomalyI_{v_1}$};
        \node[right = of sumn]  [circle, draw] (ln){$\anomalyI_{v_n}$};

        \node[below = of l1]  [circle, draw] (s1){$\anomalyII_{v_1}$};
        \node[below = of ln]  [circle, draw] (sn){$\anomalyII_{v_n}$};

        \node[below left = 2cm and 0.5cm of l1]  [circle, draw] (p1){$\observable_{v_1}$};
        \node[below left = 2cm and 0.5cm of ln]  [circle, draw] (pn){$\observable_{v_n}$};

        \node[below  = of p1]  [circle, draw] (m1){$\measurement_{v_1}$};
        \node[below  =  of pn]  [circle, draw] (mn){$\measurement_{v_n}$};
      
        \begin{pgfonlayer}{bg} 
        \node[right = 2cm of d1, draw=white] (dots) {\textbf{...}};      
        \node[below = 0.4cm of dots, draw=white] (dots2) {\textbf{...}};      
        \node[right = 2cm of p1, draw=white] (dots3) {\textbf{...}};
        \end{pgfonlayer}
      
        \draw[-{Latex[length=3mm]}] (t) -- (ht);
        \draw[-{Latex[length=3mm]}] (ht) -- (d1);
        \draw[-{Latex[length=3mm]}] (ht) -- (dn);
            
      \draw[-{Latex[length=3mm]}] (d1) -- (sum);
      \draw[-{Latex[length=3mm]}] (dn) -- (sumn);
      \draw[-{Latex[length=3mm]}] (l1) -- (sum);
      \draw[-{Latex[length=3mm]}] (ln) -- (sumn);
      \draw[-{Latex[length=3mm]}] (sum) -- (p1);
      \draw[-{Latex[length=3mm]}] (sumn) -- (pn);
      \draw[-{Latex[length=3mm]}] (s1) edge [out=270,in=20]  (m1);
      \draw[-{Latex[length=3mm]}] (sn) edge [out=270,in=20] (mn);
      \draw[-{Latex[length=3mm]}] (p1) -- (m1);
      \draw[-{Latex[length=3mm]}] (pn) -- (mn);

      \draw[-{Latex[length=3mm]}] (t) edge [out=180,in=20]  (l1);
      \draw[-{Latex[length=3mm]}] (t) edge [out=270,in=90]  (ln);
      \draw[-{Latex[length=3mm]}] (t) edge [out=180,in=35]  (s1);
      \draw[-{Latex[length=3mm]}] (t) edge [out=270,in=40]  (sn);
      \draw[-{Latex[length=3mm]}] (p1) edge [out=20,in=160]  (pn);
      \draw[-{Latex[length=3mm]}] (pn) edge [out=200,in=340] (p1);

    \end{tikzpicture}
    \caption{Simplified graphical model}
    \label{fig:used_net}
    \end{subfigure}
    \caption{Visualization of modeling of anomalies in critical infrastructure}
\end{figure*}

%A common way to model and get a better understanding of complex dynamic systems is by using a \emph{dynamic Bayesian network}~\cite{dagum1992dynamic,russell2010artificial}. Modeling conditional dependencies of a set of variables over time using a directed acyclic graph, Bayesian networks~\cite{russell2010artificial,pearl2009causality} constitute a suitable tool to model the dynamics in critical infrastructure and the influence of anomalies.
% In the case of WDNs, we can rely on the structure of the network $G$ and model pipes as edges and their connecting points as nodes.
A common way to model and obtain a better understanding of complex dynamic systems is by using a Bayesian network~\cite{russell2010artificial,pearl2009causality}. Modeling conditional dependencies of a set of variables over time is especially targeted by a variant called \emph{dynamic Bayesian networks (DBN)}~\cite{dagum1992dynamic,russell2010artificial} which use the same acyclic graph at each time point and allow additional dependencies between consecutive time points. Therefore, DBNs constitute a suitable tool to model the dynamics in critical infrastructure and the influence of anomalies.
% In the case of WDNs, we can rely on the structure of the network $G$ and model pipes as edges and their connecting points as nodes.

We will first describe a Bayesian network modeling an anomaly-free system and then discuss its extensions to scenarios containing anomalies. 
We start by describing the infrastructure as a graph $G$, e.g. edges correspond to pipes in WDNs or power lines in EGs and nodes to junctions or devices like pumps, tanks, power plants, etc.
As visualized in Figure~\ref{fig:bayes-net} at each node $v\in V(G)$ and each time step $t$ an observable quantity $\observable_{t,v}$, e.g. pressure in WDNs or voltage in EGs, describes the network state. If a sensor device is available at node $v$, we can obtain a measurement $\measurement_{t,v}$ of it. Each observable $\observable_{t,v}$ depends on the observables at time $t-1$ of the considered node and its neighborhood in $G$, i.e. $\{\observable_{t-1,v}\} \cup\{\observable_{t-1,w}|w \in N_G(v)\}$ where $N_G(v)$ is the direct neighborhood of $v$ in $G$. Besides,  $\observable_{t,v}$ additionally depends on the \emph{demand} at the respective position and time $\demand_{t,v}$. A demand might be the amount of water or power taken from the network. The demands in turn are independent of each other but each demand $\demand_{t,v}$ depends on its respective past $\demand_{t-1,v}, \demand_{t-2,v},\cdots$. Besides, all demands depend on an additional \emph{hidden cause} $\hiddencause_t$ which models outer circumstances like the weather or public holidays. The hidden cause only depends on its history, i.e. $\hiddencause_{t-1}, \hiddencause_{t-2}, \dots$.

In a demand-driven system, we can model two types of anomalies. An anomaly of type I changes the overall system dynamics, e.g. a leakage in a WDN. Modeling this kind of anomaly can be accomplished by introducing additional anomaly demands as commonly done in the water literature, e.g.~\cite{li_fast_2022}. In our network, $\anomalyI_{t,v}$ models the additional demand at every node $v \in V(G)$. Anomalies of type I only depend on the time $t$. Adding them introduces an additional dependency of the observables $\observable_{t,v}$ on $\anomalyI_{t,v}$ for each node $v$. 
In contrast, anomalies of type II influence only a single measurement. A common instance of this would be a failing or degrading sensor. %We can formalize them as a function $\anomalyII_{t,v}:\R\rightarrow\R$.
We can formalize them as an additional node $\anomalyII_{t,v}$ that depends on its own past and influences $\measurement_{t,v}$ only.%, e.g. if $\anomalyII_{t,v} = 1$ the sensor is online, e.g. $\measurement_{t,v} \mid [\anomalyII_{t,v} = 1] \sim \mathcal{N}(\observable_{t,v}, \sigma_1)$,  if $\anomalyII_{t,v} = 0$ the sensor is broken, e.g. $\measurement_{t,v} \mid [\anomalyII_{t,v} = 0] \sim \mathcal{N}(0, \sigma_0)$.\todo{added this for clearity}

Notice that though we made use of discrete time to make things easier, we can assume that $\observable_{t,v}$ is obtained by a set of differential equations and thus get a good approximation using a sufficiently small time step size.%, e.g., compare to Euler's method\todo{cite?}. %todo: Barbaras Comment: just leave the e.g. part?

A common way to express the interplay of the variables functionally is to rely on the formalism of \emph{functional models}~\cite{pearl2009causality,zhang2015estimation}. Given an underlying, directed graph with nodes $V \times \T$, the value $X_{v,t}$ of node $v \in V$ at time $t \in \T$ can be computed using a deterministic function $f_{v,t}$ that takes the values $X_{\pa(v,t)}$ of the parents and independent noise $\varepsilon_{v,t}$, i.e. we have $X_{v,t} = f_{v,t}(X_{\pa(v,t)},\varepsilon_{v,t})$~\cite{zhang2015estimation}.
Since in the considered critical infrastructure systems the underlying physics describing the interplay of the components do not change over time, in this setting assuming that the function $f_{v,t}$ does not depend on time, i.e. $f_{v,t} = f_{v,s}$ is reasonable.

Thus, we obtain a function $\observable_v : \R^{N_G(v) \cup \{v\}} \times \R \times \R \to \R$ for every node $v$ in $G$ that computes the measurements of the next time-step using the last observables and the current demand while accounting for anomalies of type I, i.e. %$\measurement_{t,v} = \anomalyII_{t,v}(\observable_v( (\observable_{t-1,w})_{w \in N_G(v) \cup \{v\}}, \demand_{t,v}+\anomalyI_{t,v}, \varepsilon_{v,t}))$. 
$\observable_{t,v} = \observable_v( (\observable_{t-1,w})_{w \in N_G(v) \cup \{v\}}, \demand_{t,v}+\anomalyI_{t,v}, \varepsilon_{v,t}))$,  
while anomalies of type II are taken care of by the measurements nodes, e.g. if $\anomalyII_{t,v} = 1$ the sensor is online, e.g. $\measurement_{t,v} \mid [\anomalyII_{t,v} = 1] \sim \mathcal{N}(\observable_{t,v}, \sigma_1)$,  if $\anomalyII_{t,v} = 0$ the sensor is broken, e.g. $\measurement_{t,v} \mid [\anomalyII_{t,v} = 0] \sim \mathcal{N}(0, \sigma_0)$. %\todo{added this for clearity}
For reasons of simplicity, we will assume that the behavior of the observables given the demands is mainly deterministic, i.e. $\observable_v$ is invariant with respect to $\varepsilon_{v,t}$ for all $v$. Notice that both observables and demand can be modeled using a single real number for each node in this setup. However, replacing those with a vector is straightforward. 

Using this setup we can now formally define the problem of anomaly detection and localization:

\begin{definition}
We say that an \emph{anomaly occurs} at node $v$ and time $t$ if $\anomalyI_{t-1,v} \neq \anomalyI_{t,v}$. The task of \emph{anomaly detection} refers to the problem of determining whether there occurs an anomaly at time $t$, i.e. determine the set $\{t \mid \exists v \in V(G) :  \anomalyI_{t-1,v} \neq \anomalyI_{t,v} \}$. The task of \emph{anomaly localization} refers to the task to determine whether there occurs an anomaly at node $v$, i.e. determine the set $\{v \mid \exists t : \anomalyI_{t-1,v} \neq \anomalyI_{t,v}\}$. 

The definition of the occurrence, detection, and localization of \emph{sensor faults} is analogous if we replace $\anomalyI$ by $\anomalyII$. 
\end{definition}

%Note that we can formulate the same for anomalies of type II.

\section{Methodology}\label{sec:methodology}
Based on the formalization of the problem we provided in the prior section, we now discuss how anomaly detection and localization can be accomplished by relying on the notion of concept drift. 
% \subsection{Leakage Detection\label{sec:meth-detection}}

\begin{figure*}[t]
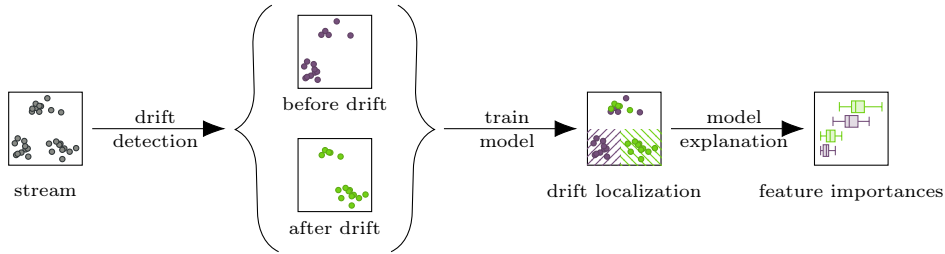

    \centering
    \tikzLocalization
    \caption{Visualization of model-based drift explanation scheme}
    \label{fig:driftexp}
\end{figure*}
\subsection{Modelling Anomalies as Concept Drift}
Considering the dynamics in the considered networks over longer periods, only the occurrence of anomalies strongly depends on the time. While demands are ever-changing they tend to follow relatively stable patterns, which will average out over longer time periods. Thus, we propose to simplify the network which we proposed in Section~\ref{sec:causal-model} (see Figure~\ref{fig:bayes-net}). %to a \emph{standard Bayesian network}~\cite{pearl2009causality,zhang2015estimation}. \tood{It is not a bayes net... me stupid}
As can be seen in Figure~\ref{fig:used_net}, we have to introduce a random variable $T$ that represents the time to model the strong time dependency of the anomalies. 
We replace $\hiddencause_t$, $\anomalyI_t$, and $\anomalyII_t$ by the  sub-networks $T \to \hiddencause$, $T \to \anomalyI$, and $T \to \anomalyII$ with $\hiddencause \mid [T = t] \sim \hiddencause_t$ and analogous for $\anomalyI$ and $\anomalyII$. And we assume that the conditional $\demand_t\mid \hiddencause_t$ is time independent so that  $\demand \mid [\hiddencause = h] \sim \demand_t\mid [\hiddencause_t = h]$.
%We replace $\hiddencause_t$ by $T \to \hiddencause$, $\demand_t\mid \hiddencause_t$ by $\demand\mid \hiddencause$ and $\anomalyI_t$ and $\anomalyII_t$ by the sub-networks $T \to \anomalyI$, $T \to \anomalyII$. 
%with the properties $\observable_{\anomalyI \mid T = t} = \observable_{\anomalyI_t}$ and $\measurement_{ \observable, \anomalyII \mid T = t} = \measurement_{\observable_t\anomalyII_t}$.

Looking at the resulting dependencies in our modeling it becomes apparent that anomalies of type I and II are the only variables that depend on time assuming the effects of the hidden causes average out as successfully shown by~\cite{vaquet2024investigating}. Besides, the anomalies influence the measurements. Thus, changes in $\anomalyI_{v}$ and $\anomalyII_{v}$ are reflected in the observed data stream. An established way to describe such a phenomenon is the notion of \emph{concept drift} or drift for shorthand. While commonly drift is defined as change in the data generating distribution~\cite{asurveyonconceptdriftadaption}, on a statistical level it makes sense to consider a general statistical interdependence of data and time:
\begin{definition}
%We say that, given a distribution on data and time $\D$ on $\T \times \X$ which decomposes into a distribution $\P_T$ in time-domain $\T$ and the conditional distributions $\D_t$ on data-space $\X$, $\D_t$ has no drift if and only if $X \indep T$~\cite{DAWIDD}.
%We say that a distribution of timed data points $(T,X)$ has drift iff data $X$ and time $T$ are dependent, i.e. a distribution $\D$ on $\T \times \R^d$ has drift iff for $(T,X) \sim \D$ we have $T \not\indep X$~\cite{DAWIDD}.
A time $\T$ indexed collection of distribution $\D_t$ on data $\X$ with observation probability in time $P_T$, i.e. $P_T$ is a distribution on $\T$ and $\D_t$ is a Markov kernel from $\T$ to $\X$, is said to have \emph{drift} iff data and time are dependent, i.e. for $T \sim P_T$ and $X \mid [T = t] \sim \D_t$ we have $T \not\!\!\indep X$~\cite{DAWIDD}.
\end{definition}
% via a distribution $\D$ on $\T \times \X$ which decomposes into a distribution $\P_T$ in time-domain $\T$ and the conditional distributions $\D_t$ on data-space $\X$~\cite{DAWIDD}. 
% One of the key findings of \cite{DAWIDD} is a unique characterization of the presence of drift as a statistical dependency of time $T$ and data $X$ if a time-enriched representation of the data $(T,X) \sim \D$ is considered, i.e. $\D_t$ has no drift if and only if $X \indep T$. 
In the remainder of this section, we discuss how this formalization of anomalies as concept drift can be used to obtain anomaly detection and localization strategies.

\subsection{Anomaly Detection\label{sec:meth-detection}}
The task of determining whether or not there is drift, referred to as \emph{drift detection}, directly aligns with detecting anomalies in the considered setup. Numerous approaches for drift detection exist in the body of literature~\cite{asurveyonconceptdriftadaption,frontiers}. Many of the state-of-the-art anomaly detection schemes implement supervised drift detection which relies on model loss as a proxy for drift in the underlying distribution. Additionally, \cite{vaquet2024investigating} experimentally showed the superior suitability of distribution-based drift detection schemes for detecting leakages in WDNs.

\subsection{Anomaly Localization}

% \subsection{Leakage Localization\label{sec:meth-loc}}
% So far, we argued that anomalies can be formalized as drift. 
% Leakage detection can be accomplished by means of drift detection, as frequently implicitly done in related work. However, as discussed before, a second mandatory step in a monitoring system is to localize the leakage, i.e. to pinpoint its exact location. 
% We will analyze anomaly localization from the perspective of drift. 
% However, before obtaining a promising methodology for leakage localization, we need to first analyze the given setup a bit further.
Next to drift detection, there are additional tools like drift localization and drift explanations that can be used to analyze the drift further~\cite{frontiers}. 
In this section, we will first discuss how anomalies of type II can be localized by employing model-based drift explanations. Afterward, we will argue why this strategy can also be applied to anomalies of type I.
% In order to perform anomaly localization, we need to analyze the effect of anomalies in the network further. We will only consider anomaly type I for now as type II only affects the observable at the same node.

\subsubsection{Type II}
To localize anomalies we need to accumulate additional knowledge on the drift. The presence of $\anomalyII_v$ only affects the measurement $\measurement_v$ but not those collected at other nodes. Thus, intuitively speaking explaining the change in the collected data, we would argue that $\measurement_v$ is the only measurement that changed, assuming that the time dependency of $\hiddencause$ on $T$ smoothes out on the considered window. A particularly suited methodology for obtaining such explanations is the model-based drift explanation scheme \cite{neucomp}. 

The main idea of model-based explanations is to reduce the problem of explaining drift to explaining a suitably trained model~\cite{esann2022,neucomp}:
As pointed out by \cite{DAWIDD,neucomp} drift is encoded in the dependence of $X$ and $T$. Thus, as visualized in Figure \ref{fig:driftexp}, by training a model $h(t \mid x)$ to estimate $\P[T = t \mid X = x]$ we essentially extract information about the drift. Given the time of the drift which can be obtained by applying drift detection as discussed in Section~\ref{sec:meth-detection}, we can train a simple binary classifier discriminating between samples from before and after the drift. While \cite{esann2022,neucomp} presented how this idea can be combined with a range of local and global explanation techniques, for detecting anomalies of type I feature-based explanation schemes are of particular relevance. 
% As we discussed that leakages mainly effect the pressure measurements which are close to the leakage location, we assume that these are the most important features when training a discriminator.
%Using that $\P[T = t \mid X = x]$ and $\P[X = x \mid T = t] = \D_t(x)$ are linked by Bayes' Theorem. 

% removed as we do not use them in the experiments
% todo additional line of thought, combined solution
% Model-based drift explanations are not limited to the case in which the exact time of the drift or the anomaly is a prerequisite. Drift segmentations~\cite{hinder2021concept} can be used to implicitly find the drift.
% As pointed out by \cite{neucomp} this can be done by multi-regression tasks aiming to fit polynomial or Fourier-transformed time \cite{MomentTrees,izbicki2017converting}. This allows the usage of a large number of different models %can be used as a model which can be explained by different explanation schemes. This 
% and might be of particular relevance in case of considerable delays of the drift detection mechanisms in realistic applications.

\subsubsection{Type I}
We will now consider anomalies of type I and argue why we can localize them by the described model-based explanations scheme as well.

One way to argue is to further simplify the DBN. Assuming that the influence of the hidden cause on the demands is comparably weak and that the measurements are accurate, Figure~\ref{fig:used_net} further simplifies. As we can incorporate the demands (and anomaly demands) into the observables we are left with $\observable_v\;v \in V(G)$ and $T$. 
Assuming an anomaly occurs at nodes $V_\anomalyI \subset V(G)$ we can compute the values at the remaining nodes without referring to $T$ so that the entire network becomes
\begin{align*}
    T \to \observable_{V_\anomalyI} \to \observable_{V(G) \setminus V_\anomalyI}.
\end{align*}
In particular, it becomes apparent that a machine learning model that reconstructs $T$ from $\observable$ does not benefit from taking $\observable_{V(G) \setminus V_\anomalyI}$ into account assuming it is provided with $\observable_{V_\anomalyI}$. Thus, if we apply model-based drift explanations using a feature-based explanation scheme it is reasonable to assume that high values are assigned to the features in $V_\anomalyI$ and low values to the features contained in ${V(G) \setminus V_\anomalyI}$. 

While this gives us an intuitive justification for employing model-based drift explanation techniques, when analyzing Figure \ref{fig:used_net} from a theoretical viewpoint there might be feedback effects as $\observable_{v_1},\dots, \observable_{v_n}$ are connected via paths through the $\demand_{v_1},\dots,\demand_{v_1}$, and $\hiddencause$. Thus, a change in one node might have global effects.

In the following, we will analyze the network dynamics to show that the effects of an anomaly of type I are particularly prominent close to the location of the anomaly and decrease as we move away from the anomalous node.

% Another way to consider the problem is to assume that an anomaly at one location has a local effect, i.e. although the effect does propagate through the system the effect strength does decrease as we move away from the node where the anomaly occurs. In this case, we can argue that the effect is particularly prominent close to the anomaly's location so that the corresponding features provide the best indication and is thus heavily used by the model. 
% To verify this idea, we need to study the network dynamics. 

As each $\observable_v$ computes the value at every single node, we can easily combine those and end up with a map that updates the measurements of the entire network $\observable : \R^{V(G)} \times \R^{V(G)} \to \R^{V(G)}, \observable_t = \observable(\observable_{t-1},\demand_t+\anomalyI_t)$, where we suppress the node-index to indicate that we consider all nodes of the same type at once, i.e. $\observable_t = (\observable_{t,v})_{v \in V(G)}$ and analogous for $\demand_t$ and $\anomalyI_t$. 
Here, $t$ represents a mere index of a time series, other than before it has no statistical meaning. 

Commonly the network dynamics in critical infrastructure are modelled by steady-state simulations, e.g.~\cite{rossman_epanet_2000}. The function $\observable$ models one step in a simulation. 
% Commonly this is done using steady-state simulations~\cite{rossman_epanet_2000}. 
Therefore, it is reasonable to believe that for the same demands, the measurements become more similar in the next time-step, i.e. $\Vert \observable(\observable_1,\demand) - \observable(\observable_2,\demand) \Vert_1 < \Vert \observable_1-\observable_2 \Vert_1$. Formally, this can be expressed by the notion of Lipschitz continuity with constant $<1$. Conversely, under the Lipschitz assumption, the dynamics in the system become a steady state. Indeed, we can show that the measurement values do not depend on initial measurements but on the demands alone, as one would intuitively expect (see appendix, Lemma~\ref{lem:bannach}).
\begin{toappendix}
\begin{lemma}
\label{lem:bannach}
Let $\demand_\cdot : \Z \to \R^{V(G)}, t \mapsto \demand_t$ be a demand pattern and assume that the transition function $\observable(\cdot, d)$ is uniformly Lipschitz continuous for all $d \in \R^{V(G)}$, i.e. we have $\sup_{d \in \R^{V(G)}} \Vert \observable(p,d) - \observable(p',d) \Vert_1 \leq C_s \Vert p-p' \Vert_1$, with constant $C_s < 1$. Then for every time-point $t_0$ it holds that $P$ has exactly one fix-point, i.e. there exists exactly one $p \in \R^{V(G)}$ such that $p = \observable(p, \demand_{t_0})$. 

Furthermore, denote by $\observable_0^{(t_0)}(p,D_\cdot) = p,\; \observable_{t+1}^{(t_0)}(p,D_\cdot) = \observable_t^{(t_0+1)}(\observable(p,D_{t_0}),D_\cdot)$. Then $\observable_t^{(t_0)}(p,D_\cdot)$ becomes independent of the choice of $p$ and $t_0$ as time passes by, i.e. for all $p,p' \in \R^{V(G)}$ and $M \in \Z_{\geq 0}$ we have
\begin{align*}
    \Vert \observable_t^{(t_0)}(p,D_\cdot) - \observable_{t+M}^{(t_0-M)}(p',D_\cdot) \Vert_1 \to 0 \text{ as } t-t_0 \to \infty.
\end{align*}
In particular, $t \mapsto \observable_t(D_\cdot) := \lim_{n \to \infty} \observable_{t+n}^{(-n)}(p,D_\cdot)$ is well defined and constant for constant demand patterns $D_\cdot = D$. 
\end{lemma}
\begin{proof}
First, prove the second statement: We apply the same argument as used in Banach's fix-point theorem. To do so first notice that we can rewrite $\observable_{t+1}^{(t_0)}(p,D_\cdot) = \observable(\observable_t^{(t_0)}(p,D_\cdot),D_{t_0+t})$ and hence
\begin{align*}
    &\Vert \observable_t^{(t_0)}(p,D_\cdot) - \observable_{t+M}^{(t_0-M)}(p',D_\cdot) \Vert_1 
    \\&= \Vert \observable(\observable_{t-1}^{(t_0)}(p,D_\cdot),D_{t_0+t-1}) - \observable(\observable_{t+M-1}^{(t_0-M)}(p,D_\cdot),D_{t_0-M+t+M-1}) \Vert_1
    \\&\leq C_s \Vert \observable_{t-1}^{(t_0)}(p,D_\cdot) - \observable_{t+M-1}^{(t_0-M)}(p',D_\cdot) \Vert_1 
    \\&\leq \cdots \leq C_s^{t-t_0} \underbrace{\Vert p-\observable_{M}^{(t_0-M)}(p',D_\cdot) \Vert_1}_{\leq K} \xrightarrow{t \to \infty} 0. 
\end{align*}

For the second statement proceed as in the proof of Banach's fix-point theorem using that as $t_0$ is fixed $\observable(\cdot, D_{t_0})$ is a contraction of $\R^{V(G)}$ onto itself.
\end{proof}
\end{toappendix}

Furthermore, it is also reasonable that a small change in demand at one node will for a single time step results in a small change of the observable at that node only, i.e. is continuous. A type of continuity that is commonly used when dealing with differential equations is $\alpha$-Hölder continuity which is given by $| \observable_v(\observable, \demand_1) - \observable_v(\observable, \demand_2) | \leq C |\demand_1-\demand_2|^\alpha$ for some constants $C > 0$ and $0 < \alpha \leq 1$. Under this assumption, we can show that the entire dynamics measured depend continuously on all input demands. Details can be found in the appendix (see Lemma~\ref{lem:hoelder}).

As a consequence, we can make statements about the mean measurements: when considering demands over a longer period of time we observe that the oscillations cancel out. Under our assumptions, this effect propagates through to the measurements. 
\begin{toappendix}
\begin{lemma}
\label{lem:hoelder}
Assume that the transition function $\observable$ is Lipschitz continuous in the first and $\alpha$-Hölder continuous second argument, i.e. $\Vert \observable(p,d) - \observable(p',d') \Vert_1 \leq C_s \Vert p-p' \Vert_1 + C_d \Vert d-d' \Vert_1^\alpha$ with $C_s < 1$. Then it holds:
\begin{enumerate}
    \item For constant demands, the pressures are $\alpha$-Hölder continuous with constant $C_d/(1-C_s)$.
    \item The pressures of a stochastic demand pattern are approximated by the pressures of the mean demand pattern up to $C_d/(1-C_s)$ times the $\alpha^\text{th}$-power of the sum of the demand wise standard deviations.
\end{enumerate}
\end{lemma}
\begin{proof}
    \textbf{1.: } Let $D,D'$ be arbitrary demands. It holds
    \begin{align*}
           \Vert \observable(D) - \observable(D') \Vert_1 
        &= \Vert \observable(\observable(D),D) - \observable(P(D'),D') \Vert_1 
        \\&\leq C_s \Vert \observable(D) - \observable(D') \Vert_1 + C_d \Vert D-D' \Vert_1^\alpha
        \\&\leq \dots \leq \underbrace{C_s^{n+1}}_{\to 0} \Vert \observable(D) - \observable(D') \Vert_1 + \underbrace{\left(\sum_{i = 0}^{n} C_s^i \right)}_{\to 1/(1-C_s)} C_d \Vert D-D' \Vert_1^\alpha
        \\&\xrightarrow{n \to \infty} \frac{C_d}{1-C_s} \Vert D-D' \Vert_1^\alpha.
    \end{align*}
    Thus for every bounded set of demands the function is locally $\alpha$-Hölder continuous. As the constant, in contrast to the rate of convergence, does not depend on the demands, the function is $\alpha$-Hölder continuous. 
    
    \textbf{2.: } Let $D : \Omega \to \R^{V(G)}$ be a stochastic demand pattern with component wise mean $\mu = \E[D]$ (this is well defined as $|V(G)|$ is finite). Then it holds
    \begin{align*}
        \Vert \observable(D(\omega)) - \observable(\mu) \Vert_1 &\overset{1.}{\leq} \frac{C_d}{1-C_s} \Vert D(\omega) - \mu \Vert_1^\alpha.
    \end{align*}
    Taking expectation on both sides and using that $0 < \alpha \leq 1$ so that $x^\alpha$ is concave on $\R_{\geq 0}$ we have by using Jensen's inequality applied twice
    \begin{align*}
        \Vert \observable(D) - \observable(\mu) \Vert_{L^1} 
        &\leq \frac{C_d}{1-C_s} \left(\E\left[\Vert D - \mu \Vert_1 \right]\right)^\alpha
        \\&= \frac{C_d}{1-C_s} \left(\sum_{v \in V(G)} \E\left[| D_v - \mu_v |\right]\right)^\alpha
        \\&\leq \frac{C_d}{1-C_s} \left(\sum_{v \in V(G)} \sqrt{\E\left[( D_v - \mu_v )^2\right]} \right)^\alpha
        \\&= \frac{C_d}{1-C_s}\left(\sum_{v \in V(G)} \text{Std}(D_v) \right)^\alpha.
    \end{align*}
\end{proof}
\end{toappendix}
Since the fluctuation usually causes some problems, data analysis is commonly performed using such mean values over time windows. In this case, we can show that the effect of anomalies decays exponentially fast throughout the network. Formally this can be stated as follows:
\begin{theorem}
\label{thm:decay}
    Assume that the transition function $O$ is Lipschitz continuous in the first and $\alpha$-Hölder continuous second argument, i.e. $\Vert O(o,d) - O(o',d') \Vert_1 \leq C_s \Vert o-o' \Vert_1 + C_d \Vert d-d' \Vert_1^\alpha$. Let $D \in \R^{V(G)}$ be a constant/mean demand pattern and $A \in \R^{V(G)}, A_w = \delta_{vw} a, \; l > 0$ be an anomaly at node $v$. 
    Denote by $O(D)$ the steady state observable for the demand $D$.
    If $C_s < 1/(\deg G+1)$, where $\deg G = \max_{u \in V(G)} |N_G(u)|$ is the maximal node degree in $G$, then the effect of the anomaly decays exponentially throughout $G$:
    \begin{align*}
        &| (O(D))_w - (O(D+A))_w | \\
        &\begin{cases}
        = 0, & d_G(w,v) = \infty \\
        < %\frac{a^\alpha C_d \cdot (C_s (\deg G + 1))^{d_G(v,w)}}{1-C_s + C_s\deg G},
        \dfrac{C_d a^\alpha \cdot (C_s (\deg G+1))^{d_G(v,w)}}{1-C_s (\deg G+1)}, & \text{otherwise}
        \end{cases}
    \end{align*}
    where $d_G(u,v)$ denotes the distance in $G$, i.e. the length of the shortest path in $G$ connecting $u$ and $v$ or $= \infty$ if there is no such path.

    The assumption $C_s < 1/\deg G$ is necessary to assure a monotonic decrease of the effect.
\end{theorem}
\begin{proof}
    Due to space restrictions, the proof can be found in the appendix.
\end{proof}
\begin{toappendix}
\begin{theorem}
    Assume that the transition function $P$ is Lipschitz continuous in the first and $\alpha$-Hölder continuous second argument, i.e. $\Vert \observable(p,d) - \observable(p',d') \Vert_1 \leq C_s \Vert p-p' \Vert_1 + C_d \Vert d-d' \Vert_1^\alpha$. Let $D \in \R^{V(G)}$ be a constant/mean demand pattern and $L \in \R^{V(G)}, L_w = \delta_{vw} l, \; l > 0$ be a leak at node $v$. 
    If $C_s < 1/(\deg G+1)$, where $\deg G = \max_{u \in V(G)} |N_G(u)|$ is the maximal node degree in $G$, then the effect of the leakage decays exponentially throughout $G$:
    \begin{align*}
        &| (P(D))_w - (P(D+L))_w | \\
        &\begin{cases}
        = 0, & d_G(w,v) = \infty \\
        < %\frac{l^\alpha C_d \cdot (C_s (\deg G + 1))^{d_G(v,w)}}{1-C_s + C_s\deg G},
        \dfrac{C_d l^\alpha \cdot (C_s (\deg G+1))^{d_G(v,w)}}{1-C_s (\deg G+1)}, & \text{otherwise}
        \end{cases}
    \end{align*}
    where $d_G(u,v)$ denotes the distance in $G$, i.e. the length of the shortest path in $G$ connecting $u$ and $v$ or $= \infty$ if there is no such path.

    The assumption $C_s < 1/\deg G$ is necessary to assure a monotonic decrease of the effect.
\end{theorem}
\begin{proof}%[Of Theorem~\ref{thm:decay}]
If $w$ and $v$ are not connected in $G$, then $d_G(w,v) = \infty$ and by the definition of the transition function, $L$ does not affect $P(D+L)$ so that the difference is $0$ and the inequality is fulfilled. 

Thus, assume $G$ is connected.
    Denote by $\Delta_u = |(P(D))_u - (P(D+L))_u|$, by $\Delta(n) = \sup_{d_G(v,u) = n} \Delta_u$, by $N(u)$ the set of all neighbours of $u$ in $G$. By convention we set $\sup \emptyset = -\infty$.
    Then we have
    \begin{align*}
        \Delta(n) &= \sup_{d_G(v,u) = n} \Delta_u
        \leq \sup_{d_G(v,u) = n} C_s \left(\Delta_u + \sum_{u' \in N(u)} \Delta_{u'} \right) + C_d \cdot \delta_{uv} l^\alpha
        \\&\leq C_s (\deg G +1) \max\{\Delta(n-1),\Delta(n),\Delta(n+1)\} + \1[n = 0] C_d l^\alpha,
    \end{align*}
    where the last inequality holds since every node is affected by its at most $\deg G$ neighbors and itself and each of those is at most 1 node closer or further away from $v$ than $u$.

    Now, show that $\Delta(n)$ is decreasing: For $n \geq 1$ we have 
    \begin{align*}
        \Delta(n)&\leq C_s(\deg G+1) \max\{ \Delta(n-1),\Delta(n),\Delta(n+1) \}
    \end{align*}
    As $C_s (\deg G +1) < 1$ the maximum on the right-hand side is bigger than $\Delta(n)$, in particular, it has to be given by either $\Delta(n+1)$ or $\Delta(n-1)$. If we apply this argument inductively, starting at $N = \sup_{u \in V(G)} d_G(u,v)$ so that $\Delta(N+1) = -\infty$ we see that
    \begin{align*}
        \Delta(N) \leq \Delta(N-1) \leq \dots \leq \Delta(2) \leq \Delta(1) \leq \Delta(0),
    \end{align*}
    as all $\Delta(n) \geq 0$ and 
    where the last term is obtained by applying the argument to $n = 1$.
    Thus, by induction, we have
    \begin{align*}
        \Delta(n) \leq (C_s (\deg G+1))^n \Delta(0).  
    \end{align*}
    In particular, equality for $\Delta(1)$ holds if and only if $\Delta(0) = 0$.

    For $n = 0$ we have that each neighbour hast distance $1$ from $v$ thus we have
    \begin{align*}
        \Delta(0) &\leq C_s \Delta(0) + C_s \deg G \Delta(1) + C_d l^\alpha 
        \\&< C_s (\deg G+1) \Delta(0) + C_d l^\alpha 
        %\\\Rightarrow \Delta(0) &\leq \frac{C_d l^\alpha}{(1-C_s (\deg G+1))}
        \\\Rightarrow \Delta(n) &< \frac{C_d l^\alpha (C_s (\deg G+1))^n}{1-C_s (\deg G+1)}.
    \end{align*}
    Here the strict inequality follows by the fact $\Delta(0) \geq C_dl^\alpha > 0$.

\textbf{Necessity: } Consider $G = (\{v,w,u_1,\dots,u_n\},\{v,w\} \times \{u_1,\dots, u_n\})$, $\observable_v(p,d) = d$, $\observable_{u_i}(p,d) = C_s p_v$, $\observable_w(p,d) = C_s \sum_i p_{u_i}$. Then $|\observable_v(p,d)-\observable_v(p',d')| \leq |d-d'|$, $|\observable_{u_i}(p,d) - \observable_{u_i}(p',d')| = C_s |p_v-p_v'|$, $|\observable_w(p,d)  - \observable_w(p',d')| = C_s \left|\sum_{i = 1}^n p_{u_i}-p_{u_i}' \right| \leq C_s \sum_{i = 1}^n \left|p_{u_i}-p_{u_i}' \right|$ so the transition function fulfills the requirements. However, 
\begin{align*}
    (\observable(D))_{u_i} &= C_s (\observable(D))_{v} = C_s D_v \\
    (\observable(D))_w &= C_s \sum_{i = 1}^n (\observable(D))_{u_i} = n C_s^2 D_v 
\end{align*}
and thus
\begin{align*}
    \frac{|(\observable(D))_{w} - (\observable(D+L))_{w}|}{|(\observable(D))_{u_i} - (\observable(D+L))_{u_i}|} &= \frac{n C_s^2 l}{C_s l} = n C_s
\end{align*}
as $\deg G = n$ the statement follows.
\end{proof}
\end{toappendix}
Considering the explanation scheme used to localize anomalies of type II aims at finding the most important measurement nodes to discriminate between data collected before and after the anomaly occurred, we see that this strategy translates to anomalies of type I. As their effect decays throughout the network, nodes close to the anomaly will be more important for inference.

%One might argue that the assumption of constant demands is not realistic, but as stated before, such can usually be used to approximate the behavior of the network sufficiently well. 
%The relevance of this result thus becomes apparent: The differences in the mean of pressure values taken before and after the occurrence of a leak increase the closer we get to the leak. Thus, they allow for a leak localization. However, there are also more advanced strategies which we will consider in the next section.

\subsubsection{Baseline Methodology for Anomaly Localization\label{sec:baseline}}
Next to justifying model-based explanations, our considerations yield a simpler methodology: 
Combining our findings from modeling anomalies in critical infrasturcture (Figure~\ref{fig:used_net}) and Theorem~\ref{thm:decay}, we obtain that leakages can be modeled as a drift in the system which inflicts itself locally much stronger than globally as the effect of the leakage on the pressure measurements decays exponentially. This observation gives rise to a straightforward window-based leakage localization technique:
\begin{align}
    v_{\text{leakage}}(t; w) := \underset{v \in V(G)}{\argmax} \left| \sum_{i = t-w}^{t} \observable_{i,v} - \sum_{i = t}^{t+w} \observable_{i,v} \right|\label{eq:mean}
\end{align}
where $t$ is the onset time of the leakage, which can be obtained by drift detection, and $w$ is the window size, a hyper-parameter of this method.

While this simple methodology is justified by the theory we expect some shortcomings when applying it to realistic data. For example, just considering the discrepancy in the mean does not account for differences in the variance across different sensor nodes.

In the remainder of the paper, we will experimentally evaluate the suitability of model-based drift explanations. We will consider different instantiations of feature-based explanation schemes.

\section{Experiments\label{sec:exp}}
Before presenting and discussing the results we will briefly introduce the experimental setup and evaluation\footnote{The code of the experiments will be made available after acceptance.}.

\begin{table*}[ht]
\caption{Results of anomaly detection in WDNs. Type I: different metrics using topological distances. The table shows the median (m), mean ($\mu$), and standard deviation ($\sigma$) over all 764 nodes and 42 runs. Bigger better ($+$), smaller better ($-$). Type II: recall, precision, and F1-score for the localization of sensor faults over all experiments.}
\footnotesize
\centering
\results{src/results_water_with_acc.csv}
\label{tab:results}
\end{table*}

\subsection{Setup, Data, and Metrics} 
\subsubsection{Case Studies}

To evaluate the proposed methodology, we consider two exemplary instantiations of critical infrastructure. For both we rely on state-of-the-art simulation tools as this way we can simulate different types of anomalies throughout the network.

One popular network used in the domain of WDNs is the L-Town network which resembles parts of the old town of Limassol, Cyprus. It is a comparably complex and realistic water supply network. As realistic demands for a period of one year are available, one can simulate different leakage and sensor fault scenarios using the atmn package. As visualized in Figure~\ref{fig:ltwon}, Area A contains 661 nodes and 764 edges with 29 optimally placed pressure sensors~\cite{vrachimis_battle_2022}. 

To obtain expressive results we simulate several scenarios for the experimental evaluation. We consider leakage sizes of 7mm, 11mm, 15mm, and 19mm which would be categorized as small background to medium-sized leakages in the L-Town network. Especially, the smaller leakages pose challenges to many approaches in the literature~\cite{vrachimis_battle_2022}. To create the scenarios, 23 windows of one month length at a 15-minute sampling interval are chosen with 15 days offset between the windows. For each window 10 random leak onset times are generated with the leak persisting to the end of the window. This leads to 230 scenarios per leak size and pipe despite the limited amount of available realistic demands. Besides, we consider a range of realistic sensor faults modeled in \cite{reppa_sensor_2016} applied in the same manner to evaluate the localization of faults of type II.

As a second case study, we consider an EG based on the SimBench urban LV grid \cite{meinecke2020simbench}. It is a three-phase electrical distribution grid with a nominal voltage of 400V consisting of 58 nodes arranged in 7 strings of varying lengths. As measurement points, we consider a set of 5 nodes being placed in highly connected areas. Again, rely on simulations and use the Matlab simulink simscape library.
% \todo[inline]{LARS: add a description and experiment specs, Stichpunkte reichen mir. VALERIE: kann hier nicht auch auf unser paper verwiesen werden, da da ähnliche sensor faults untersucht wurden im elektrische Netz? } 
% \begin{itemize}
%     \item based on SimBench urban LV grid \cite{meinecke2020simbench} 
%     \item 58 nodes three phase electrical distribution grid with nominal voltage of 400V
%     \item supplied by an 630kVA three phase transformer
%     \item modelled in Matlab simulink using simscape library
%     \item consists of 7 stings of several lengths
%     \item crutial nodes to measure 2,3,4,6,9
%     \item sensor faults occour on \todo[inline]{Jons: welche hast du als sensor faults angenommen?}
% \end{itemize}
In this setup, we consider sensor failures only. We again rely on the modeling by \cite{reppa_sensor_2016}.

\subsubsection{Used Methods}
As visualized in Figure \ref{fig:driftexp}, in our experiments, we assume that a drift detector already determined the time of the drift. As baselines, we consider a random baseline (random), choosing a pipe at random, and the window-based strategy derived in section \ref{sec:baseline}. In addition to the statistical test-based leakage localization strategy proposed in \cite{vaquet2024investigating} (KS), we use the following instantiations of model-based drift explanations:
We combine tree-based models (RandomForests (RF) and ExtraTrees (ET)) with feature importance (FI) and Permutation Feature Importance (PFI). Besides, we consider Logistic Regression with l2 regularization (LogReg), parameters are automatically determined by cross-validation, and linear SVM (SVM)). Here we consider the absolute value of the weight vector as the importance score.

\subsubsection{Metrics}
\newcommand{\numCloser}{\#closer}
\newcommand{\relDist}{rel. dist.}
\newcommand{\bestThree}{best-3}
For localizing anomalies of type I, we consider the metrics defined in \cite{vaquet2024investigating}. 
In the following let $s^* \in S$ denote the selected node and $v \in V(G)$ the leaky node. 
We use the following four metrics: 
\begin{itemize}
    \item distance between selected and actual node (distance; $d(s^*,v)$),
    \item number of nodes in $S$ closer to the actual node (\numCloser; $|\{s \in S \mid d(s,v) < d(s^*,v)\}|$),
    \item relative distance between actual node, selected and optimal node (\relDist; $d(s^*,v) / \min_{s \in S} d(s, v)$, we exclude the cases where $v \in S$ to avoid division by 0) %which is normalized in contrast to the simple distance and smooth in contrast to the closer node metric, 
    \item and the intersection size of the 3 most important and closest sensors (\bestThree) which is relevant for interpolation tasks.
\end{itemize}
% distance between selected and actual node (distance; $d(s^*,v)$), number of nodes in $S$ closer to the actual node (\numCloser; $|\{s \in S \mid d(s,v) < d(s^*,v)\}|$), and relative distance between actual node, selected and optimal node (\relDist; $d(s^*,v) / \min_{s \in S} d(s, v)$, we exclude the cases where $v \in S$ to avoid division by 0) which is normalized in contrast to the simple distance and smooth in contrast to the closer node metric, and the intersection size of the 3 most important and closest sensors (\bestThree) which is relevant for interpolation tasks.

For $d$ we consider the following distances: the topological distance (topo.) in the graph $G$, i.e. the number of nodes of the shortest connecting path, and the actual geographic distance (geo.). We also documented model accuracy. 
% Note that for the experiments on locating sensor faults, we only document the first two metrics as $\relDist=d$ since anomalies only occur at sensor nodes. For this reason, $\bestThree$ not reported either.

As considering distances is not a sensible choice for evaluating the localization of sensor faults (type II), in this case, we consider the recall, precision, and F1 score of the overall localization capabilities of all experiments, e.g. the localization only counts as successful if the malfunctioning sensor has been identified.

\subsection{Results}

The results of our experiments with WDNs are summarized in Table~\ref{tab:results}.
%\paragraph{Description}

\paragraph{WDNs -- Type I}
When analyzing the obtained results for localizing leakages, we found that the values when computed topological and geographical distance differ neither quantitatively, i.e. in the absolute number obtained, nor qualitatively, i.e. concerning the ranking of the methods. The only exception is the distance metric, which does differ quantitatively only, as to be expected.
Therefore, we will not distinguish the cases for the distance measures and only present the results for the topological distance.%\footnote{The full results for geographical distance as well as additional visualization of all metrics are provided in the appendix.} 

% and Figure~\ref{fig:results_metric}.
As can be seen, there is a strong correlation between distance, \numCloser, and \relDist, which we could confirm by a range of statistical tests. 
%An overview over all metrics is presented in Table~\ref{tab:results} and Figure~\ref{fig:results_metric}. 
% As can be seen there is a strong correlation between distance, \numCloser, and \relDist. If the methods are ignored this correlation is very strong (Spearman's $\rho$ absolute $\approx 0.9$ at $p \ll 10^{-4}$, and similar for Pearson CC and Kendall's $\tau$), if methods are taken into account the correlation with the absolute distance decreases in particular for the tree based methods with drift detection (Spearman's $\rho$ absolute $\approx 0.6$ at $p \ll 10^{-4}$), the correlation between \numCloser and \relDist is mainly unaffected. In particular, the results for distance and \relDist are very similar to the one presented in Figure~\ref{fig:results_metric:close:bp} and~\ref{fig:results_metric:close:nem}.
% For \bestThree globally there is a correlation between the other metrics, but is is comparably weak (Spearman's $\rho$ absolute $\approx 0.5$ at $p \ll 10^{-4}$, and similar for Pearson CC and Kendall's $\tau$) and vanishes if single methods are considered.
We find that the simple mean baseline performs significantly better than the random baseline as suggested by Theorem~\ref{thm:decay}. We make similar observations for the KS test-based method proposed by \cite{vaquet2024investigating}. However, many of the explanation-based leakage localization schemes are capable of pinpointing the leakage even closer. An explanation for this is that the linear models compensate for variance and covariance terms. To see this, consider the two mean values of the mean approach as prototypes that induce a linear model, which typically suffers from this problem. 

While in the median all tree-based methods find the closest sensor node, the experiments with logistic regression tend to find sensors that are slightly further away from the leakage. %This ranking can also nicely be seen in the box plots in Figure~\ref{fig:results_metric:close:bp} and be confirmed by statistical tests. 
This might be a reflection of the fact that linear models are unable to learn the more complex, non-linear relationships between features which explains that the tree-based models outperform them if information about the drift is provided. 

% Besides, we find that the experiments without drift detection extract relevant information about the drift, although with less precision. This is to be expected as the training involves the additional complication of recognizing the drift. However, we consider the results to be very promising. 

Considering the best-3-metric, we find that even the best performing tree-based implementations are only identifying less than 2 of the closest sensor nodes when considering the top 3 feature importances. This needs further investigation in future work.

In addition to the global investigation, figuring out whether there are areas that pose particular challenges to the methodology is interesting and might generate insights. Therefore we additionally created a visualization of the results for some selected methods 
%We also visualized the results for some methods 
(see Figure~\ref{fig:maps}). The maps show the distance between the leakage and the selected sensor node. The values are pipe-wise normalized taking on values between the theoretically closest and the mean for the Random baseline. As can be seen, the error distribution is not uniform but rather depends on the location in the network. Though it is influenced by the model (and the metric as further analysis shows) certainty tendencies and error hot spots can be observed across all setups. On closer consideration, we can attribute at least some of those effects to the properties of the network. For example, the water is supplied to the network in the upper right of the map causing small leakages to remain hidden in the strong inflow. However, we were yet not able to explain all areas with low localization scores.
Thus, further research on our agnostic approach is necessary.
%However, it shows the limitations of a fully uninformed and agnostic approach like ours and shows the necessity of further research. 

% \begin{figure}[t]
%     \centering
%     \begin{minipage}{0.48\textwidth}
%         \includegraphics[width=\textwidth]{src/img/results/boxplog/best_3_length}
%         \subcaption{\bestThree-metric.}
%     \end{minipage}
%     \hfill
%     \begin{minipage}{0.48\textwidth}
%         \includegraphics[width=\textwidth]{src/img/results/boxplog/num_closer_length}
%         \subcaption{\numCloser-metric.}
%         \label{fig:results_metric:close:bp}
%     \end{minipage}
%     \caption{Visualisation of results (on mean per pipe) for \bestThree- and \numCloser-metric using geographic distances. Every point corresponds to the mean over all runs on a single pipe. Ordering is according to mean (best to worst: left to  right).}
%     \label{fig:results_metric}
% \end{figure}

\begin{figure}[t!]
    \centering
    \begin{minipage}{0.3\textwidth}
        \includegraphics[width=\textwidth]{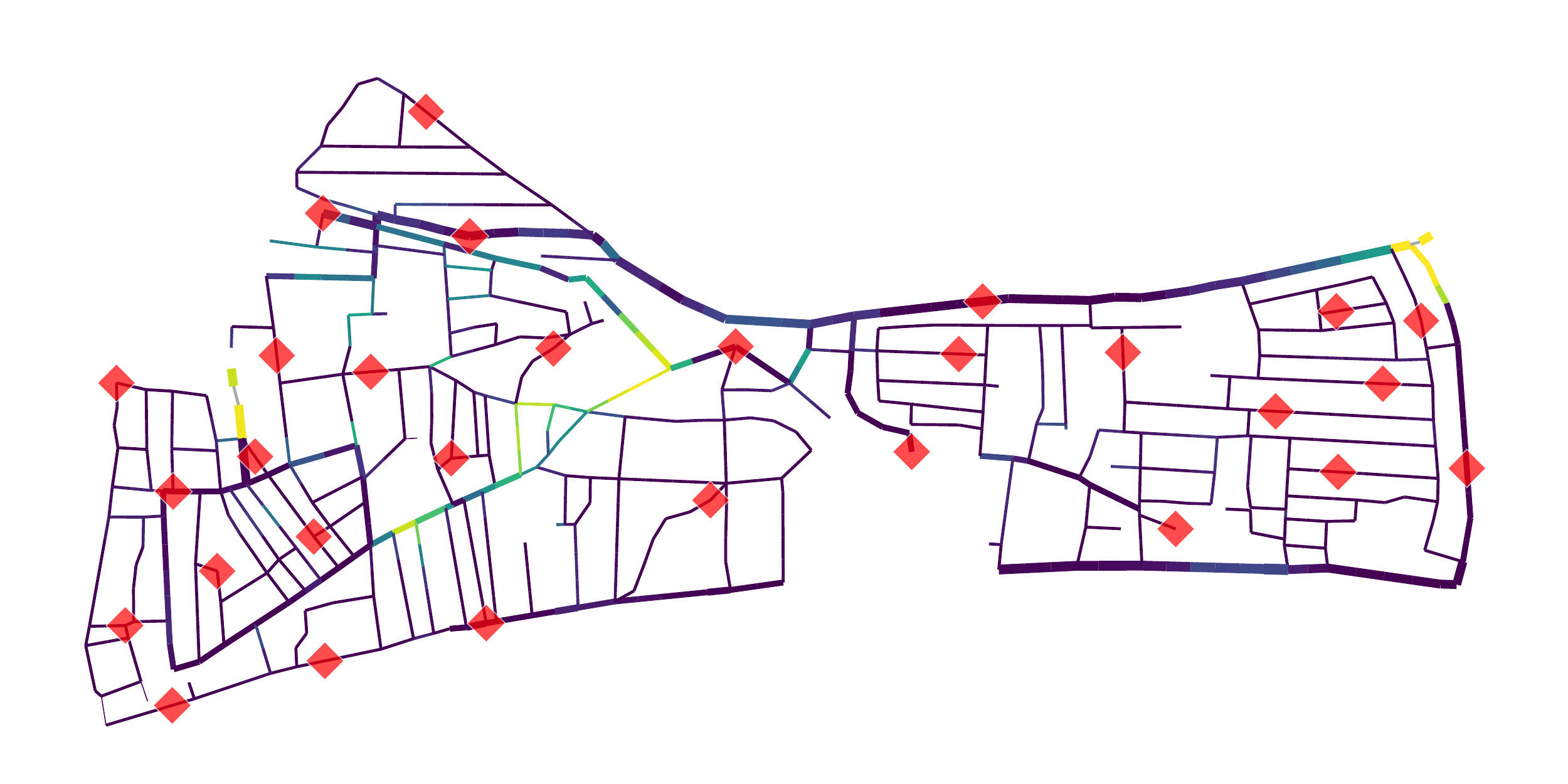}
        \subcaption{FI (ET)}
    \end{minipage}
    \begin{minipage}{0.3\textwidth}
        \includegraphics[width=\textwidth]{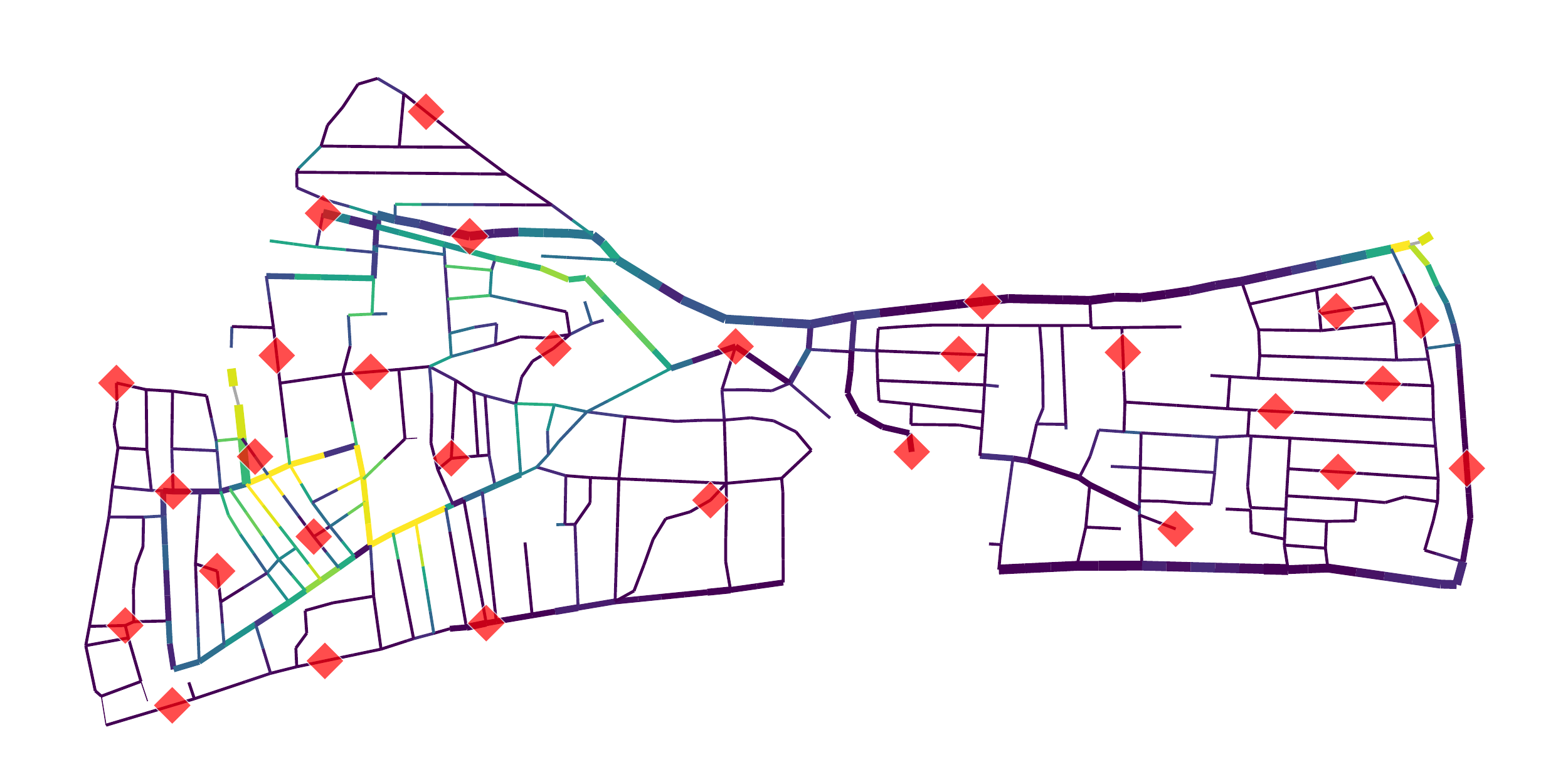}
        \subcaption{SVM}
    \end{minipage}
    \begin{minipage}{0.3\textwidth}
        \includegraphics[width=\textwidth]{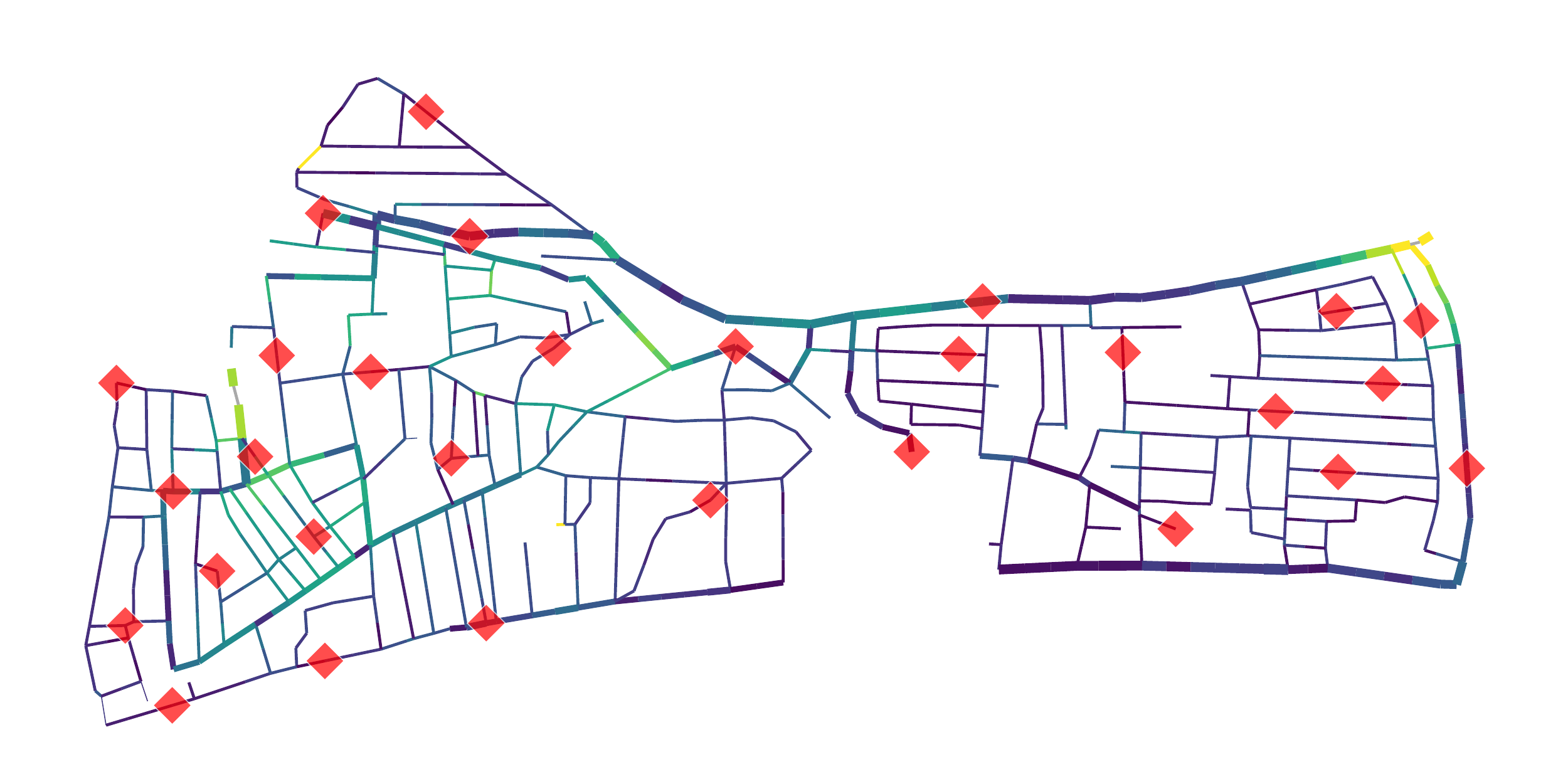}
        \subcaption{Mean}
    \end{minipage}
    \begin{minipage}{0.05\textwidth}
    \includegraphics[height=10em]{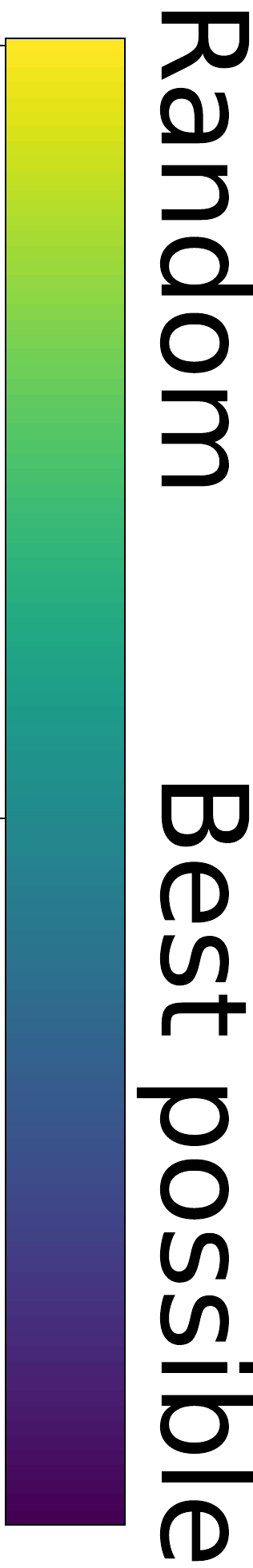}
    \end{minipage}
    \caption{Visualisation of (relative) distance between pipe and selected node. Value is normalized between best possible value (closest sensor; purple) and Random baseline (yellow). Red diamonds mark sensor positions.}
    \label{fig:maps}
\end{figure}

\paragraph{WDNs -- Type II}
Considering our experiments on localizing sensor faults, we observe good localization capabilities for many of the considered approaches. Again, the mean and KS baseline perform significantly better than the random baseline. Employing the tree-based explanation scheme we observe near to perfect localization results. However, we observe that the versions implementing linear models are not capable of localizing anomalies reliably at all.

\paragraph{EGs -- Type II}
\begin{table}
\caption{Results of sensor fault localization in EGs.}
\footnotesize
\centering
\resultsElecAcc{src/results_elec.csv}
\label{tab:results}
\end{table}
Our results on sensor fault localization in EGs are visualized in Table \ref{tab:results}. While we overall obtain much worse localization results, we observe similar overall behavior: while the linear models do not succeed at localizing the faults, the tree-based models show that the localization works significantly better than the random baseline. We did not expect very good results as the considered EG represents a much finer spatial resolution, i.e. each node represents one household while in the considered WDN, each node represents multiple ones resulting in much smoother signals. Thus, for this benchmark, we obtained much noisier signals which increase the difficulty for models to find the drifting features.

\section{Discussion\label{sec:discussion}}
In this contribution, we focused on anomaly localization in critical infrastructure, a task of particular societal relevance, as for example, drinking water is an increasingly limited resource due to climate change.
To the best of our knowledge, this work is the first work analyzing the dynamics of critical infrastructure utilizing Bayesian networks. We obtained that concept drift is a valid way to model anomalies in critical infrastructure and that their influence decays exponentially fast. Based on these insights we obtained an approach relying on model-based drift explanations which showed promising first results in our experimental evaluation which was mainly conducted in water distribution networks. In contrast to established work leakage localization, our methodology has two main advantages. First, it is independent of the network's topology and does not require any additional data except for pressure measurements. Thus, it is applicable in many real-world scenarios where data is limited and can generalize over different network topologies. Second, it is computationally lightweight in comparison to hydraulic-based approaches which rely on fitting critical parameters and running many computationally expensive simulations at inference time.
The proposed modeling and the obtained strategy to localize leakages by model-based drift explanations generalizes to other types of anomalies and other critical infrastructure instances that can be modeled as graphs as showcased by our experiments on sensor faults in EGs.

\bibliographystyle{IEEEtran}
\bibliography{bib_water}
% \printbibliography
% \begin{thebibliography}{00}
% \bibitem{b1} G. Eason, B. Noble, and I. N. Sneddon, ``On certain integrals of Lipschitz-Hankel type involving products of Bessel functions,'' Phil. Trans. Roy. Soc. London, vol. A247, pp. 529--551, April 1955.
% \bibitem{b2} J. Clerk Maxwell, A Treatise on Electricity and Magnetism, 3rd ed., vol. 2. Oxford: Clarendon, 1892, pp.68--73.
% \bibitem{b3} I. S. Jacobs and C. P. Bean, ``Fine particles, thin films and exchange anisotropy,'' in Magnetism, vol. III, G. T. Rado and H. Suhl, Eds. New York: Academic, 1963, pp. 271--350.
% \bibitem{b4} K. Elissa, ``Title of paper if known,'' unpublished.
% \bibitem{b5} R. Nicole, ``Title of paper with only first word capitalized,'' J. Name Stand. Abbrev., in press.
% \bibitem{b6} Y. Yorozu, M. Hirano, K. Oka, and Y. Tagawa, ``Electron spectroscopy studies on magneto-optical media and plastic substrate interface,'' IEEE Transl. J. Magn. Japan, vol. 2, pp. 740--741, August 1987 [Digests 9th Annual Conf. Magnetics Japan, p. 301, 1982].
% \bibitem{b7} M. Young, The Technical Writer's Handbook. Mill Valley, CA: University Science, 1989.
% \end{thebibliography}
% \vspace{12pt}
% \color{red}
% IEEE conference templates contain guidance text for composing and formatting conference papers. Please ensure that all template text is removed from your conference paper prior to submission to the conference. Failure to remove the template text from your paper may result in your paper not being published.

\end{document}